%% file: neurips_2026.tex
\documentclass{article}

\PassOptionsToPackage{numbers}{natbib}

 \usepackage[main, final]{neurips_2026}

\usepackage[utf8]{inputenc} 
\usepackage[T1]{fontenc}    
\usepackage{hyperref}       
\usepackage{url}            
\usepackage{booktabs}       
\usepackage{amsfonts}       
\usepackage{amsmath}
\usepackage{mathtools}
\usepackage{nicefrac}       
\usepackage{microtype}      
\usepackage{xcolor}         
\usepackage{graphicx}
\usepackage{tikz}
\usepackage{subcaption}
\usepackage[capitalize]{cleveref}
\usepackage{overpic}
\usepackage{amsthm, amssymb}
\usepackage{wrapfig}
\usepackage{placeins}

\usepackage{pifont}
\newif\ifshowlinks
\showlinkstrue
\newif\iffulllinks
\fulllinkstrue
\newif\iflabeledlinks
\labeledlinkstrue
\newif\ifaffillogos
\affillogostrue
\newcommand{\affilicon}[1]{\ifaffillogos\raisebox{-0.25\height}{\includegraphics[height=1.1em]{#1}}\hspace{0.15em}\fi}
\newcommand{\linkicon}[1]{\raisebox{-0.2\height}{\includegraphics[width=1em,height=1em]{#1}}}

\usetikzlibrary{shapes, arrows, positioning, calc}

\definecolor{pastelPurple}{RGB}{235, 215, 255}
\definecolor{pastelGreen}{RGB}{220, 255, 235}
\definecolor{pastelBlue}{RGB}{210, 240, 255}
\definecolor{circleBorder}{RGB}{120, 180, 220} 

\definecolor{sapienzabordeaux}{RGB}{130, 36, 51}
\definecolor{sapienzapetrolgreen}{RGB}{0, 103, 120}

\definecolor{donorFT}{HTML}{4ab191}
\definecolor{donorQAT}{HTML}{8c84e0}
\definecolor{donorQV}{HTML}{C026D3}
\definecolor{PT}{HTML}{808080}
\definecolor{PTQ}{HTML}{e86f6e}
\definecolor{cream}{HTML}{faf9f5}

\newtheorem{proposition}{Proposition}
\newtheorem{theorem}{Theorem}
\newtheorem{lemma}{Lemma}

\hypersetup{
    pdftitle={Zero-Shot Quantization via Weight-Space Arithmetic},
    pdfauthor={Daniele Solombrino, Antonio Andrea Gargiulo, Alessandro Zirilli, Luca Zhou, Adrian Robert Minut, Emanuele Rodol\`a}
}

\tikzset{
    block/.style={
        rectangle,
        draw=none,
        rounded corners=7pt,
        minimum height=1.7cm,
        minimum width=2.4cm,
        align=center,
        font=\small\sffamily,
        text=black
    },
    checkpoint/.style={block, fill=pastelPurple},
    ptq/.style={block, fill=pastelGreen},
    decision/.style={block, fill=pastelPurple, minimum width=2.7cm},
    QAT/.style={block, fill=pastelGreen},
    deploy/.style={block, fill=pastelBlue},
    container/.style={
        circle,
        draw=circleBorder,
        thin,
        transparent, 
        fill=none,
        opacity=0,
        inner sep=3pt,
        text opacity=1,
        font=\bfseries\small\sffamily,
        text=black
    },
    arrow/.style={
        ->,
        >=stealth,
        thin,
        draw=black
    },
    annot/.style={
        align=center,
        font=\small\sffamily,
        text=black
    },
    annot_arrow/.style={
        ->,
        >=stealth,
        draw=circleBorder,
        thick
    }}
\title{Zero-Shot Quantization via Weight-Space Arithmetic}

\author{%
    \parbox{\dimexpr\linewidth-2\tabcolsep\relax}{%
    \centering
    Daniele Solombrino$^{1}$ \quad Antonio Andrea Gargiulo$^{1}$ \quad Alessandro Zirilli$^{1}$ \\ \vspace{.25em}Luca Zhou$^{1}$ \quad Adrian Robert Minut$^{1}$ \quad Emanuele Rodolà$^{1,2}$
    }%
    \\
    \parbox{\dimexpr\linewidth-2\tabcolsep\relax}{\small \centering
        \vspace{1em}
        $^1$\begin{tabular}[t]{@{}l@{}}\href{https://gladia.di.uniroma1.it/}{\affilicon{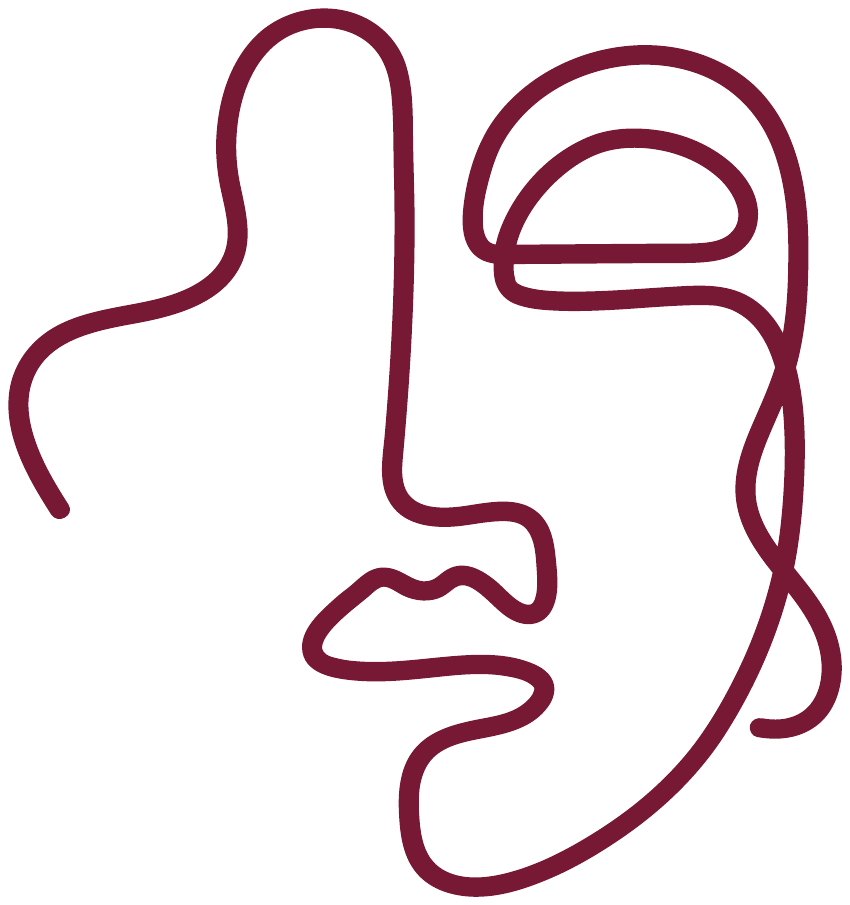}GLADIA Research Lab}\\\href{https://www.uniroma1.it/}{\affilicon{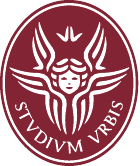}Sapienza University of Rome}\end{tabular} \quad
        $^2$\begin{tabular}[t]{@{}l@{}}\href{https://paradigma.inc/}{\affilicon{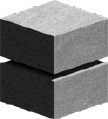}Paradigma}\end{tabular}
    }
}

\begin{document}

\maketitle

\begin{abstract}
We show that robustness to post-training quantization (PTQ) is a transferable direction in weight space. We call this direction the \emph{quantization vector}: extracted from a donor task by simple weight-space arithmetic, it can be used to patch a receiver model and improve PTQ Top-1 accuracy by up to \(60\)$\%$ in a 3-bit setting, without receiver-side quantization-aware training (QAT). 
When using no receiver validation data, the method provides a zero-shot, low-cost alternative to QAT for extremely low-bit deployment. Across multiple vision and language models and more than $30$ tasks, donor quantization vectors often yield substantial gains even when donor and receiver tasks differ markedly. We further prove rigorously that quantization vectors are well-defined and do not suffer from reparameterization symmetries, and provide a local geometric account of their effect. Together, these results suggest that quantization robustness can be partially isolated, reused, and transferred through simple weight-space algebra.
\ifshowlinks
\begin{center}
\iflabeledlinks
\vspace{0.6em}
\makebox[\linewidth][c]{\fontsize{7.1}{8.4}\selectfont
\begin{tabular}{@{}l@{\hspace{0.5em}}l@{}}
\linkicon{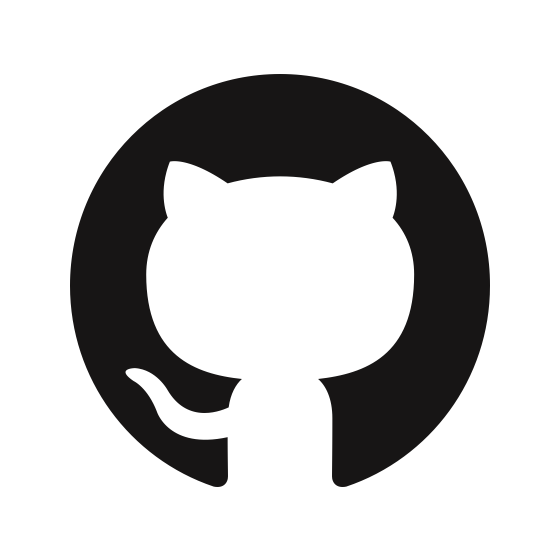}\hspace{0.1cm}Code: & \href{https://github.com/gladia-research-group/qat-transfer}{https://github.com/gladia-research-group/qat-transfer} \\[0.15em]
\linkicon{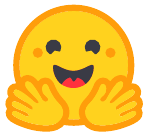}\hspace{0.1cm}timm vision checkpoints: & \href{https://huggingface.co/gladia/qat-transfer-timm}{https://huggingface.co/gladia/qat-transfer-timm} \\[0.15em]
\linkicon{figures/icons/huggingface.pdf}\hspace{0.1cm}OpenCLIP vision checkpoints: & \href{https://huggingface.co/gladia/qat-transfer-open_clip}{https://huggingface.co/gladia/qat-transfer-open\_clip} \\[0.15em]
\linkicon{figures/icons/huggingface.pdf}\hspace{0.1cm}Text checkpoints: & \href{https://huggingface.co/gladia/qat-transfer-text}{https://huggingface.co/gladia/qat-transfer-text} \\[0.15em]
\linkicon{figures/icons/huggingface.pdf}\hspace{0.1cm}VAE checkpoints: & \href{https://huggingface.co/gladia/qat-transfer-vae}{https://huggingface.co/gladia/qat-transfer-vae} \\[0.15em]
\raisebox{-0.05\height}{\ding{105}}\hspace{0.1cm}Research graph: & \href{https://flywheel.paradigma.inc/node/de486ca0-8e83-4041-bbf6-0b0be565435b}{https://flywheel.paradigma.inc/node/de486ca0-8e83-4041-bbf6-0b0be565435b} \\
\end{tabular}}
\else
\iffulllinks
\linkicon{figures/icons/github.png}\hspace{0.1cm}\href{https://github.com/gladia-research-group/qat-transfer}{https://github.com/gladia-research-group/qat-transfer}\\[0.3em]
\linkicon{figures/icons/huggingface.pdf}\hspace{0.1cm}\href{https://huggingface.co/gladia}{https://huggingface.co/gladia/}qat-transfer-\{\href{https://huggingface.co/gladia/qat-transfer-timm}{timm}, \href{https://huggingface.co/gladia/qat-transfer-open_clip}{open\_clip}, \href{https://huggingface.co/gladia/qat-transfer-text}{text}, \href{https://huggingface.co/gladia/qat-transfer-vae}{vae}\}\\[0.3em]
\raisebox{-0.05\height}{\ding{105}}\hspace{0.1cm}\href{https://flywheel.paradigma.inc/node/de486ca0-8e83-4041-bbf6-0b0be565435b}{https://flywheel.paradigma.inc/node/de486ca0-8e83-4041-bbf6-0b0be565435b}
\else
\linkicon{figures/icons/github.png}\hspace{0.1cm}\href{https://github.com/gladia-research-group/qat-transfer}{Code}\quad\textbar\quad
\linkicon{figures/icons/huggingface.pdf}\hspace{0.1cm}\href{https://huggingface.co/gladia/qat-transfer-timm}{timm}\,/\,\href{https://huggingface.co/gladia/qat-transfer-open_clip}{OpenCLIP}\,/\,\href{https://huggingface.co/gladia/qat-transfer-text}{Text}\,/\,\href{https://huggingface.co/gladia/qat-transfer-vae}{VAE} checkpoints\quad\textbar\quad
\raisebox{-0.05\height}{\ding{105}}\hspace{0.1cm}\href{https://flywheel.paradigma.inc/node/de486ca0-8e83-4041-bbf6-0b0be565435b}{Research graph}
\fi
\fi
\end{center}
\fi
\end{abstract}

\section{Introduction}
Deep neural networks are typically trained and stored in high-precision floating-point formats, which impose substantial memory footprint and bandwidth demands during deployment. 
Integer quantization addresses this bottleneck by representing model parameters with extremely low-bit integers, reducing storage and inference costs. 
Among the available quantization strategies, post-training quantization (PTQ) \citep{nagel2021white, gholami2022survey} is particularly attractive because it can be applied to an already trained model without requiring additional optimization or access to the original training data.
However, at extremely low bit-widths, PTQ can severely distort the learned parameters, causing a marked degradation in downstream task performance.

Quantization-aware training (QAT) \citep{jacob2018quantization} mitigates this issue by exposing the model to quantization effects during optimization, allowing the parameters to adapt to low-bit perturbations.
As a result, QAT often recovers a large fraction of the accuracy lost by naive PTQ.
The drawback is that it must be run separately for each receiver task, increasing the data, compute, and engineering cost of deployment.
This creates a practical trade-off between efficiency and robustness: PTQ is cheap but fragile at very low bit-widths, whereas QAT is robust but costly.
This raises a natural question: can the robustness learned through QAT on one task be isolated and \emph{reused} on another?
If so, one could improve low-bit performance on a receiver model without receiver-side quantization training.

We investigate this question through the lens of weight-space arithmetic \citep{ilharco2022editing}.
Given a donor task with both standard and QAT checkpoints, we define the \emph{quantization vector} (QV) as the displacement between them.
We then patch a receiver checkpoint by adding the scaled donor QV.
The central hypothesis is that QAT leaves behind a reusable displacement toward a more quantization-favorable region of weight space; as illustrated in Figure~\ref{fig:killer_figure}, donor QV patching is intended to occupy the middle ground between vanilla PTQ and full receiver-side QAT.

\begin{wrapfigure}[23]{r}{6.5cm}
\vspace{-0.15cm}
    \centering
    \begin{overpic}[width=.9\linewidth]{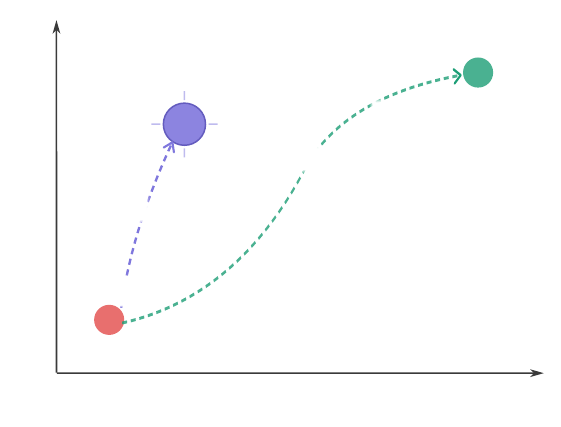}
    \put(10,5){\fontsize{7}{8}\selectfont $2.59 \times 10^{-4}$}
    \put(78,5){\fontsize{7}{8}\selectfont $2.86 \times 10^{4}$}
    \put(30,0){\fontsize{7}{8}\selectfont Computational cost (TFLOPs)}
    \put(0.5,12){\fontsize{7}{8}\selectfont 25\%}
    \put(0,66){\fontsize{7}{8}\selectfont 97\%}
    \put(3,22){\fontsize{7}{8}\selectfont \rotatebox{90}{3-bit Top-1 accuracy}}
    \put(81,67){\fontsize{7}{8}\selectfont \textcolor{donorFT}{QAT}}
    \put(22,60){\fontsize{7}{8}\selectfont \textcolor{donorQAT}{QV-patched}}
    \put(15,23){\fontsize{7}{8}\selectfont \textcolor{PTQ}{PTQ}}
    \put(14,37.5){\fontsize{7}{8}\selectfont \textcolor{donorQAT}{add donor QV}}
    \put(36,46.5){\fontsize{7}{8}\selectfont \textcolor{donorQAT}{$\theta_{\mathcal{R} \leftarrow \mathcal{D}} = \theta_{\mathcal{R}} + \lambda \rho_{\mathcal{D}}$}}
    \end{overpic}
    \caption{\textbf{Zero-shot QV patching.} A donor quantization vector $\rho_{\mathcal{D}} \coloneq \theta_{\mathcal{D},\text{QAT}} - \theta_{\mathcal{D}}$, extracted as the weight-space displacement between a standard fine-tuned donor checkpoint and its QAT counterpart, is added to a receiver checkpoint to obtain the patched model $\theta_{\mathcal{R} \leftarrow \mathcal{D}} = \theta_{\mathcal{R}} + \lambda \rho_{\mathcal{D}}$. The plot is not to scale but uses true numbers; it illustrates the intended operating regime of our method, namely improving low-bit accuracy over PTQ without paying the full receiver-side cost of QAT.\label{fig:killer_figure}}
\end{wrapfigure}

Two theoretical questions arise immediately.
First, is QV itself a legitimate vector in a common parameter space, or can fine-tuning inject hidden symmetries that make coordinate-wise subtraction ambiguous?
Second, once a donor QV is well-defined, what determines whether it transfers successfully to a receiver?

We address both aspects in this paper.
We show that quantization vectors are generically well-defined: continuous linear reparameterization symmetries are broken by the training dynamics, while permutation symmetries occur only on a measure-zero set of common initializations.

We then show that donor transfer admits a simple local theory: under a quadratic model of the receiver's post-quantization objective, optimal donor patching becomes a projection problem in the receiver's geometry: the optimal scale is the projection coefficient onto the donor direction, and the recoverable fraction of the receiver's own QAT gain is exactly a cosine-squared alignment term between donor and receiver QVs.
This gives a concrete explanation for two empirical patterns that recur throughout the paper: some donors transfer much better than others, and the scaling factor \(\lambda\) can determine whether the same donor direction helps or hurts.
This perspective complements recent geometry-based views of QAT and weight-space arithmetic \citep{ilharco2022editing, zhou2025task, tabesh2025cage}.
However, rather than modeling QAT as a task-specific optimization procedure that must be repeated from scratch, we ask whether its net effect can be isolated as a reusable displacement and transferred across tasks.

We evaluate this framework across multiple ViT scales \citep{dosovitskiy2020image}, text encoders, vision, and text tasks.
Empirically, donor QVs often improve receiver robustness to 3-bit PTQ, and the gains can be large even when donor and receiver tasks differ substantially and no receiver validation data is used to select the best scaling factor for the QV.
Moreover, the role of the scaling factor predicted by the theory is clearly visible in practice: unit-scale transfer is already often effective, while tuning \(\lambda\), when receiver validation data is available, removes much of the destructive interference that appears when a fixed step size is imposed across all donor--receiver pairs.

The full codebase is available on \href{https://github.com/gladia-research-group/qat-transfer}{GitHub}, including machine-checked Lean 4~\citep{demoura2021lean4} proof artifacts for our theoretical results.
We additionally release model checkpoints on the Hugging Face Hub for all model families considered in our experiments, organized into collections for learned-head vision models, CLIP-style vision models, text models, and VAE decoders~\citep{gladia_qat_transfer_timm_2026,gladia_qat_transfer_openclip_2026,gladia_qat_transfer_text_2026,gladia_qat_transfer_vae_2026}, including the staged-schedule QAT and QAD donors of Appendix~\ref{app:variants}.
Finally, we make the associated research graph available on \href{https://flywheel.paradigma.inc/node/de486ca0-8e83-4041-bbf6-0b0be565435b}{Flywheel}.

Our contributions are fourfold:
\begin{itemize}
    \item We introduce \emph{quantization vectors}, weight-space displacements between matched standard and QAT checkpoints that operationalize QAT-induced robustness as a transferable object.
    \item We rigorously prove that quantization vectors are \emph{well-defined}: coordinate-wise subtraction does not generically suffer from hidden symmetry injection during fine-tuning; we also extend the result to ordinary task vectors outside the quantized regime.
    \item We propose a \emph{zero-shot} cross-task patching framework and a local geometric theory of donor transfer: under a quadratic model of the receiver objective, the best donor patch recovers a cosine-squared fraction of the receiver's own QAT gain.
    \item Across multiple vision/text tasks and model scales, we show that QV patching often substantially improves 3-bit PTQ performance over vanilla PTQ, that scale calibration makes transfer markedly more reliable, and how the method’s benefits vary across bit-widths and stronger quantizers.
\end{itemize}

\section{Related Work}
\subsection{Loss Landscape Geometry \& Low-Bit Quantization Robustness} 
Quantization has long been studied to reduce the memory and computational costs of neural network inference, with post-training quantization and quantization-aware training as the two main branches. 
PTQ is attractive because it can be applied after standard training, does not require access to the original training data, and often uses a small set of calibration data.
However, its accuracy can deteriorate sharply at very low bit-widths (3 bits and below).
QAT addresses this limitation by incorporating quantization effects during optimization, typically yielding better low-bit performance at the expense of additional data and compute requirements \citep{ptqQATsurvey}.

This trade-off becomes particularly challenging in ViTs \citep{dosovitskiy2020image}, where PTQ suffers due to non-standard activation distributions. 
As a result, many methods have focused on specialized quantizer design, calibration rules \citep{ptq4vit, repqvit, adalog}, or data-free approximation strategies such as MimiQ \citep{mimiq} and DFQ-ViT \citep{dfqvit}. Despite all these advancements, such approaches often still require optimization. 
%
%
Recent literature suggests that the effect of quantization is tightly linked to the local geometry of the loss landscape \citep{catalan2025training}.
Hessian-based methods show that curvature can predict sensitivity to low-precision perturbations \citep{dong2019hawq,dong2020hawq}, while other studies \citep{nahshan2021lapq} have demonstrated that aggressive low-bit quantization can induce highly non-smooth optimization landscapes.
In ViTs, recent studies have further emphasized the irregularity of quantized loss surfaces and the role of geometry in determining quantization difficulty \citep{frumkin2023jumping}.

On the training side, \citet{liu2021saq} connects quantization to sharpness-aware optimization.
Instead, recent works argue that QAT, or quantization-induced noise, can guide optimization toward flatter minima or a lower Hessian norm \citep{wang2022squat,javed2024qtdog,catalan2025training}, and other weight perturbations applied during training can induce other structural properties in the weights \citep{tallini2026weight}.
Another work by \citet{tabesh2025cage} further formalizes QAT as a multi-objective optimization problem that seeks a Pareto-optimal point between task loss minimization and quantization constraints, achieved through a curvature-aware correction term derived from the local Hessian.
Our method is fully consistent with these views; it differs only in its purpose: rather than optimizing a task-specific training recipe, we ask whether the \emph{displacement} induced by such geometry-aware adaptation can be isolated as a reusable direction in parameter space and transferred across tasks.

\subsection{Arithmetic in Weight Space}
Recent studies indicate that neural networks trained from a common initialization often remain in a compatible basin \citep{wortsman2022model}. Task Arithmetic \citep{ilharco2022editing} defines task vectors (TVs) as the difference between a fine-tuned and a pretrained model, showing that they can be used to modify model behavior. Generally, this line of work treats coordinate-wise differences as well-defined as long as the initialization is shared. As an additional contribution, we state this assumption and provide a proof in our setting.
Similarly, \citet{cai2023robust} uses weight-space directions to inject robustness to input corruptions. Other methods improve weight-space arithmetic by finding optimal combinations of TVs \citep{yang2024adamerging}, mitigating sign disagreement \citep{ties}, randomly dropping updates \citep{yu2024language}, using evolutionary strategies \citep{sakana, mencattinimerge}, or working at the layer level to better respect network structure \citep{stoica2024model, gargiulo2025task, Iso-C}.

We extend this perspective to quantization: instead of encoding a \textit{task}, our QV encodes the \textit{structural robustness} needed to survive low-precision quantization. Our method also differs from \citet{quantizedtaskvectors} where task vectors are quantized to achieve cheaper model merging; instead, in this manuscript, we define \emph{non-quantized displacements to improve quantization robustness}.
To our knowledge, this is the first work to isolate the displacement induced by QAT as a reusable, zero-shot patch for cross-task transfer of quantization robustness.

\section{Background}

In this section, we review the foundational concepts necessary to formalize our framework: a fixed \emph{fake-quantization} operator, working at low-bit precision, and vector arithmetic, operating in weight-space, between checkpoints that share a common pretrained initialization.

\subsection{Symmetric Per-Channel Weight Quantization}

In this section we focus on weights-only  quantization, as the quantity we transfer later represents a displacement in parameter space. Consider a linear layer with a weight matrix \(W \in \mathbb{R}^{d_{\mathrm{out}} \times d_{\mathrm{in}}}\). In the setting of this paper, each output channel of \(W\) is quantized independently using a signed symmetric integer grid. For a given bit-width \(b\), the range of representable integers is defined as:
\begin{equation}
q_{\min} = -2^{b-1},
\qquad
q_{\max} = 2^{b-1} - 1.
\end{equation}
For channel \(i\), we set the scale based on the largest absolute weight in that channel and quantize by:
\begin{equation}
s_i = \frac{\max_j |W_{ij}|}{q_{\max}},
\qquad
\widehat{W}_{ij} = \mathrm{clip}\!\left(\left\lfloor \frac{W_{ij}}{s_i} \right\rceil, q_{\min}, q_{\max}\right).
\label{eq:symmetric_quantization_formula}
\end{equation}
This quantization process ensures that the weights are scaled appropriately, with values clipped to a specified minimum and maximum.
To convert the integer tensor back to floating point, we use the dequantization process defined as follows:
\begin{equation}
\widetilde{W}_{ij} = s_i \widehat{W}_{ij},
\qquad
\mathrm{FQ}(W) = \widetilde{W}.
\label{eq:quantize_dequantize_roundtrip}
\end{equation}
Here \(\lfloor \cdot \rceil\) denotes rounding to the nearest integer, while the \(\mathrm{clip}\) function truncates values to fit the representable range. By extension, \(\mathrm{FQ}(\theta)\) refers to the checkpoint obtained by applying fake quantization (FQ) to every quantized linear weight tensor in \(\theta\), while leaving the remaining parameters unchanged.
Note that per-channel quantization commutes with permutations \(g\), i.e. $\mathrm{FQ}(g\cdot W) = g\cdot \mathrm{FQ}(W)$. This desirable property will be instrumental in the theorems that follow.

The specific quantizer matters for our methodology. A different bit-width, granularity, or quantization rule could lead to varying perturbations during QAT, and consequently results in different weight-space displacements. We focus on the symmetric per-channel case because is the operator employed both to create donor QAT checkpoints and to evaluate patched receivers. Throughout the paper, \(\mathrm{FQ}\) specifically refers to 3-bit symmetric per-channel weight quantization.
During QAT, the same fake-quantization operator is incorporated into the forward pass while training, meaning the model is optimized under the perturbation induced by Eq.~\ref{eq:symmetric_quantization_formula}. Since rounding and clipping are not differentiable operations, gradients are approximated with the straight-through estimator \citep{bengio2013estimating,jacob2018quantization}. 
In the next section, we will isolate the parameter displacement between a standard fine-tuned checkpoint and a checkpoint trained to withstand this same fake-quantization noise.

\subsection{Weight-Space Arithmetic}
The second component of our framework is weight-space arithmetic, which studies linear operations between models that share a common initialization. These checkpoints often remain within a compatible basin, allowing their parameter differences to be composed.
Let \(\theta_{\text{pre}}\) denote a pretrained checkpoint and let \(\theta_{\text{task}}\) be the result of fine-tuning it for a specific downstream task. Task arithmetic \citep{ilharco2022editing} represents this adaptation by the displacement
\begin{equation}
\tau_{\text{task}} = \theta_{\text{task}} - \theta_{\text{pre}}\,,
\label{eq:task_vector_equation}
\end{equation}
which is referred to as a \emph{task vector} (TV).
When checkpoints share the same architecture and initialization, TVs are treated as vectors in a common parameter space. They can be added to other compatible checkpoints, producing meaningful and controllable changes in model behavior \citep{wortsman2022model}.
Our method adopts this viewpoint but focuses on a different source of variation. Instead of measuring the change from pretraining to task adaptation, we assess the change \emph{from standard fine-tuning to QAT on the same task} (as we will show in Eq. \eqref{eq:quantization_vector} below).

\paragraph{Well-definedness of the TV.}
%
%
It is worth noting that, while the difference $\theta_{\text{task}} - \theta_{\text{pre}}$ is algebraically valid, fine-tuning could in principle inject hidden \emph{symmetries} (e.g., neuron permutations) that would place the two checkpoints in different gauges. To our surprise, no theoretical result in the broad task arithmetic literature rigorously rules this out, so we set out to fill this gap.
In Theorem \ref{thm:adamw_signed_permutation} (Appendix \ref{app:wd}), we prove that, in the ViT regime studied here, this potential issue is generically absent as AdamW dynamics break continuous linear reparameterization symmetries. We prove this also holds true under QAT fine-tuning, and furthermore, mismatched permutation gauges between matched full-precision and QAT endpoints happen with probability zero, as they occur only on a measure-zero set of common initializations. This ensures that our construction below (in particular Eq. \eqref{eq:quantization_vector}) is well-defined in the ambient backbone parameter space.

\section{Method}

The core idea behind our method is to treat robustness to low-bit quantization not as a per-task training-time constraint that must be relearned for each downstream model, but rather as a {\em transferable} geometric alignment in weight space of a fixed architecture. 
We approach this in two stages: first, we define the quantization vector as the difference between the matched full-precision and QAT checkpoints, and then we describe how to transfer this vector to a receiver model.

\subsection{Quantization Vector}
\label{sec:qv}

Let \(\mathcal{D}\) denote a \textit{donor} dataset, and let \(\theta_{\mathcal{D},\text{QAT}}\) and \(\theta_{\mathcal{D}}\) be the checkpoints obtained by fine-tuning the same pretrained backbone \(\theta_{\text{pre}}\) on \(\mathcal{D}\) with and without QAT, respectively. 
Both fine-tunings additionally share the same random head initialization through a common training seed.
\footnote{We write \(\theta\) for the shared backbone coordinates, since all quantization and vector operations act only on the backbone. 
}

We define the \emph{quantization vector} (QV) as the displacement given by the equation:
\begin{equation}
    \rho_{\mathcal{D}} = \theta_{\mathcal{D},\text{QAT}} - \theta_{\mathcal{D}}\,.
    \label{eq:quantization_vector}
\end{equation}

By comparing Eqs.~\ref{eq:task_vector_equation} and~\ref{eq:quantization_vector}, we can see that this definition of the QV parallels the notion of TVs in weight-space arithmetic; however, the two displacements represent different effects.
The TV in Eq.~\ref{eq:task_vector_equation} captures the change in parameters from a pretrained model to a task-specific solution. In contrast, the QV in Eq.~\ref{eq:quantization_vector} represents the change from a standard fine-tuned solution to one that is more robust to low-bit quantization. Our hypothesis is that this displacement isolates \emph{structural robustness} to quantization noise without substantially altering the learned task behavior.

This motivates a geometric view of QAT: it moves parameters toward a solution that better balances task performance and quantization robustness. This is consistent with recent work framing QAT as a multi-objective problem \citep{tabesh2025cage}; however, unlike that work, we do not seek to model or improve the optimization process; instead, we focus on isolating its net effect in weight space and study the \emph{transferability} of quantization awareness. The QV is our operational representation of this effect.

\paragraph{Well-Definedness of The QV.}
As anticipated, a potential concern with any displacement in weight space is gauge mismatch: two functionally comparable checkpoints may differ by hidden parameter symmetries, which can make coordinate-wise subtraction ambiguous. In our setting, this issue is generically absent, as demonstrated by the following theorem.

\begin{theorem}[Quantization vectors are generically well-defined]
\label{thm:qv_well_defined}
Fix the donor task \(\mathcal{D}\) and any realization of the stochastic components of training. Let \(\Phi_\mathcal{D}^{\mathrm{FT}}\) and \(\Phi_\mathcal{D}^{\mathrm{QAT}}\) denote the resulting finetuning and its corresponding QAT endpoint maps on the trainable backbone parameters. Assume that the only exact symmetries preserved by the regime form a finite permutation group \(\mathcal G\), and that \(\Phi_\mathcal{D}^{\mathrm{FT}}\) and \(\Phi_\mathcal{D}^{\mathrm{QAT}}\) are piecewise real-analytic. If no nontrivial \(g \in \mathcal G\) identifies the two maps on an open set, then for almost every common initialization \(\theta_{\mathrm{pre}}\), the matched endpoints
\[
\theta_\mathcal{D} = \Phi_\mathcal{D}^{\mathrm{FT}}(\theta_{\mathrm{pre}}),
\qquad
\theta_{\mathcal{D},\mathrm{QAT}} = \Phi_\mathcal{D}^{\mathrm{QAT}}(\theta_{\mathrm{pre}})
\]
lie in the same gauge. Consequently,
\[
\rho_\mathcal{D} = \theta_{\mathcal{D},\mathrm{QAT}} - \theta_\mathcal{D}
\]
is a well-defined parameter-space displacement.
\end{theorem}

\noindent\emph{Proof sketch.}
Appendix~\ref{app:wd} shows that AdamW admits no nontrivial continuous linear equivariances, so the only surviving exact gauge is discrete; because we use per-channel quantization, the fake quantizer (FQ) of our QAT commutes with that same finite permutation gauge; and the event that the finetuning (FT) and its matched QAT land in different gauges is the zero set of a nontrivial piecewise analytic function, hence measure zero.

Theorem~\ref{thm:qv_well_defined} justifies treating Eq.~\eqref{eq:quantization_vector} as an ordinary vector difference in the ambient backbone parameter space. Proposition~\ref{prop:alignment_controls_transfer} below studies the complementary question of when a well-defined donor QV transfers successfully to a receiver.

\subsection{Zero-Shot Patching}

Let \(\mathcal{R}\) denote a \textit{receiver} task for which we only possess a standard checkpoint \(\theta_{\mathcal{R}}\) trained without QAT.
Given a donor QV \(\rho_{\mathcal{D}}\), we construct a patched receiver, with improved resilience to PTQ-induced noise, by adding the donor displacement to the receiver checkpoint:
\begin{equation}
    \theta_{\mathcal{R} \leftarrow \mathcal{D}} = \theta_{\mathcal{R}} + \lambda \rho_{\mathcal{D}} \,.
    \label{eq:qv_patching}
\end{equation}
Here, \(\lambda \in \mathbb{R}\) is a scaling coefficient modulating the intensity of the robustness patch.
Intuitively, this operation moves the receiver parameters in a direction that was previously learned to improve robustness to PTQ noise on the donor task.

This protocol assumes that the receiver and donor share the same pretrained backbone initialization.
%
Task-specific heads may differ across tasks and are not part of the transferred vector.
Under this assumption, the corresponding backbone checkpoints lie within a compatible loss basin of the weight space \citep{wortsman2022model}, making linear transfer meaningful.
The resulting procedure is zero-shot and data-free on the receiver side, since it requires \emph{no access to \(\mathcal{R}\)'s training data}, relying entirely on the pre-computed weight-space arithmetic.
\begin{wrapfigure}[23]{r}{6.5cm}
\centering
    \begin{overpic}[width=.95\linewidth,trim=20 0 0 0]{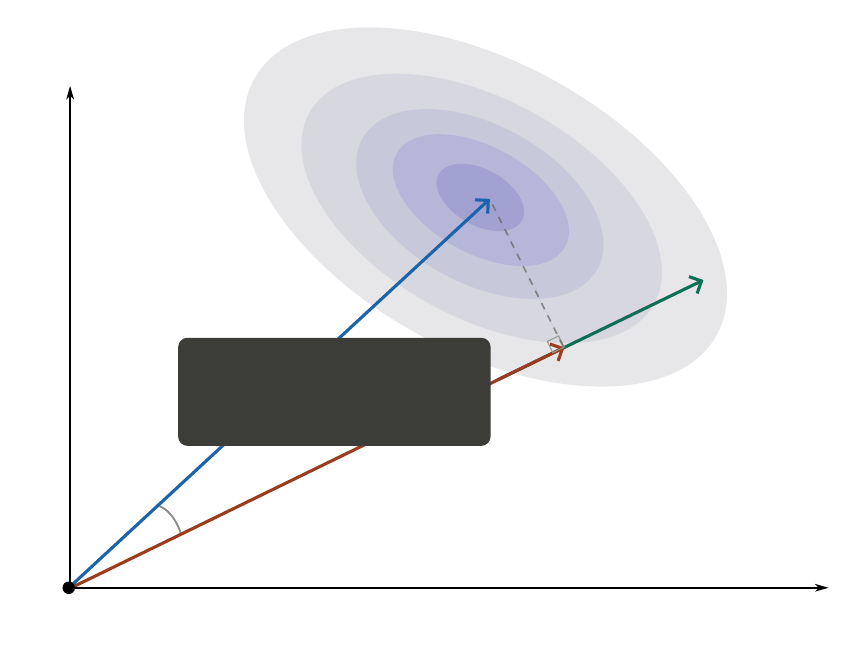}
    \put(79,42){\fontsize{7}{8}\selectfont $\rho_\mathcal{D}$}
    \put(79,38){\fontsize{7}{8}\selectfont (donor QV)}
    \put(56,62){\fontsize{7}{8}\selectfont $\rho_\mathcal{R}$}
    \put(56,58){\fontsize{7}{8}\selectfont (\emph{ideal} receiver QV)}
    \put(0,4.5){\fontsize{7}{8}\selectfont $\theta_\mathcal{R}$}
    \put(94,5){\fontsize{7}{8}\selectfont $w_1$}
    \put(-2.2,65.5){\fontsize{7}{8}\selectfont $w_2$}
    \put(63.5,33){\fontsize{7}{8}\selectfont $\lambda^\ast \rho_\mathcal{D}$}
    \put(63.5,29){\fontsize{7}{8}\selectfont (optimal patch)}
    \put(17,18){\fontsize{7}{8}\selectfont $\gamma$}
    \put(24,34.3){\fontsize{7}{8}\selectfont \color{cream}Recovered gain}
    \put(26,29.8){\fontsize{7}{8}\selectfont \color{cream}$=\cos^2_{H_\mathcal{R}}\gamma$}
    \end{overpic}
    \caption{\textbf{Geometric view of donor patching.} The blue vector \({\rho}_{\mathcal R}\) is the receiver's own backbone QV (\emph{unknown} in our setting), the green ray is the donor direction \({\rho}_{\mathcal D}\), and the red vector \(\lambda^\star {\rho}_{\mathcal D}\) is the orthogonal projection of \({\rho}_{\mathcal R}\) onto that donor line. Proposition~\ref{prop:alignment_controls_transfer} states that the fraction of receiver-side QAT gain recovered by this best donor patch is exactly \(\cos^2\gamma\), where \(\gamma\) is the angle between \({\rho}_{\mathcal D}\) and \({\rho}_{\mathcal R}\). Illustration in whitened coordinates where the Hessian \(H_\mathcal{R}\) defines the local geometry.\label{fig:donor}
    }
\end{wrapfigure}
\paragraph{Geometric interpretation.}
The patching rule in Eq.~\ref{eq:qv_patching} admits a simple local interpretation that helps explain two empirical patterns that will recur in Section~\ref{sec:results}: first, some donor vectors transfer much better than others; second, the scaling factor \(\lambda\) can affect if a transfer is helpful or harmful. 

The next proposition makes this precise under a local quadratic model of the receiver's post-quantization objective.

\begin{proposition}[Alignment controls donor transfer]
\label{prop:alignment_controls_transfer}
Let \(g_{\mathcal R}(\delta)\) denote the receiver's post-quantization objective after a displacement \(\delta\), and assume that near the receiver's own quantization vector \(\rho_{\mathcal R}\) the objective is smooth enough (locally quadratic):
\begin{align*}
g_{\mathcal R}(\delta)
=
g_{\mathcal R}(\rho_{\mathcal R})
+
\frac12 \|\delta-\rho_{\mathcal R}\|_{H_{\mathcal R}}^2,
\\
\|u\|_{H_{\mathcal R}}^2 := u^\top H_{\mathcal R} u,
\qquad
H_{\mathcal R}\succ 0.
\end{align*}
For a donor quantization vector \(\rho_{\mathcal D}\), the best scaled donor patch is
\[
\theta_{\mathcal R} + \lambda^\star \rho_{\mathcal D},
\qquad
\lambda^\star
=
\frac{\rho_{\mathcal D}^\top H_{\mathcal R}\rho_{\mathcal R}}
     {\rho_{\mathcal D}^\top H_{\mathcal R}\rho_{\mathcal D}}.
\]
Moreover, the fraction of the receiver-side QAT gain recovered by this best donor patch is exactly
\[
\cos_{H_{\mathcal R}}^2(\rho_{\mathcal D},\rho_{\mathcal R}),
\qquad
\cos_H(u,v)
:=
\frac{u^\top Hv}
{\sqrt{u^\top Hu}\sqrt{v^\top Hv}}.
\]

Therefore:
\begin{enumerate}
    \item donor transfer helps if and only if \(\rho_{\mathcal D}\) is not \(H_{\mathcal R}\)-orthogonal to \(\rho_{\mathcal R}\);
    \item donor transfer matches receiver-side QAT if and only if \(\rho_{\mathcal D}\) is a nonzero scaling of \(\rho_{\mathcal R}\);
    \item under this local quadratic model, the best scaled donor patch cannot outperform the ideal receiver QV.
\end{enumerate}
\end{proposition}

\noindent\emph{Proof.} See Appendix~\ref{app:proof_cosine_squared_recovery}.

Proposition~\ref{prop:alignment_controls_transfer} turns donor patching into a projection problem in the receiver's local geometry; see Figure~\ref{fig:donor}. The optimal scale \(\lambda^\star\) is the projection coefficient of the receiver QV onto the donor direction, and the recoverable fraction of receiver-side QAT gain is exactly the squared cosine of the angle between the two in the \(H_{\mathcal R}\)-geometry (i.e. all inner products are weighted by \(H_\mathcal{R}\)).


This theoretical result captures the second-order term of the local transfer geometry, thus it is exact for a purely quadratic local model. One can further show that the same picture persists beyond that idealization: if the receiver objective is \emph{locally smooth} with Lipschitz Hessian near \(\rho_{\mathcal R}\), then the cosine-squared law remains accurate up to a cubic remainder. We refer to Appendix~\ref{sub:first_deviation_in_local_displacement} for the complete result with proof.

\begin{figure}
    \centering
    \includegraphics[width=1.0\linewidth]{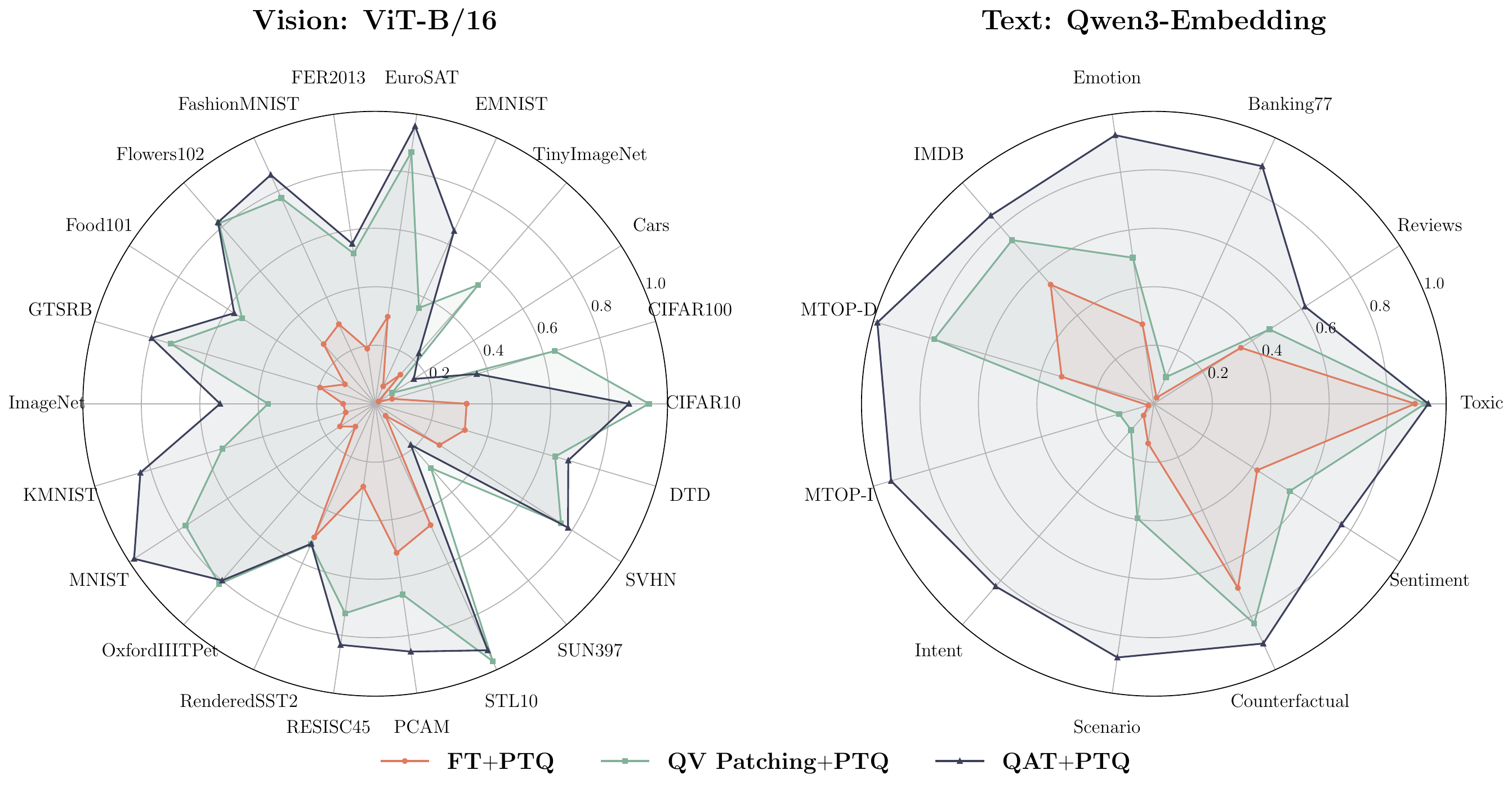}
    \caption{\textbf{QV patching across vision and text classifiers.} Receiver Top-1 accuracy under 3-bit PTQ for vanilla fine-tuned checkpoints (FT+PTQ), QV-patched checkpoints (QV Patching+PTQ), and receiver-side QAT checkpoints (QAT+PTQ) on \texttt{\texttt{ViT-B/16}} and \texttt{Qwen3-Embedding}.}
    \label{fig:timm_supervised_automodelforsequenceclassification_qv_patching}
\end{figure}
\subsection{Evaluation}
To assess the transferability of the QV across different tasks, we compare the downstream performance of the patched receiver checkpoint against the PTQ performance of the unpatched counterpart.
Specifically, we apply fake quantization to both models using Eq.~\ref{eq:quantize_dequantize_roundtrip} and measure Top-1 accuracy on the receiver task. 
For a donor-receiver pair \((\mathcal{D},\mathcal{R})\), we define the transfer gain as:
\begin{equation}
    \Delta(\mathcal{D},\mathcal{R}) = \mathrm{Acc}(\mathrm{FQ}(\theta_{\mathcal{R} \leftarrow \mathcal{D}})) - \mathrm{Acc}(\mathrm{FQ}(\theta_{\mathcal{R}})) \,.
    \label{eq:cross_task_transfer_metric}
\end{equation}
Here, \(\mathrm{Acc}(\cdot)\) denotes the Top-1 accuracy of a model on the evaluation set of the receiver task, and \(\mathrm{FQ}(\cdot)\) denotes the fake-quantization operator from Eq. \ref{eq:quantize_dequantize_roundtrip}, which applies 3-bit symmetric per-channel weight quantization to the model parameters.

For checkpoints trained with QAT, evaluation is always performed after removing the training-time QAT wrappers and applying the same PTQ operator used throughout the paper. This approach ensures that comparisons between FT, PTQ, and QAT are made under common inference-time conditions.
The transfer gain, denoted as \(\Delta\), provides a clear metric for assessing transfer success: a positive \(\Delta\) indicates that the transfer successfully improves robustness. 
In this context, the donor QV effectively mitigates PTQ-induced performance degradation in the receiver task. 
Conversely, a negative \(\Delta\) suggests that patching is harmful: the donor QV causes destructive interference in the receiver's parameter space, worsening performance relative to vanilla PTQ. 
A value of \(\Delta\) close to zero implies that patching has a negligible impact: the donor QV neither improves nor degrades the receiver's resilience to extremely low-bit PTQ. See Figures \ref{fig:baselines_timm_supervised_baseline_bar_all_models}, \ref{fig:baselines_open_clip_baseline_bar_all_models} and \ref{fig:automodelforsequenceclassification_baseline_bar_all_models} for detailed baseline performance.

\section{Experimental Setup}
\label{sec:exp_setup}
We evaluate QV patching on a broad benchmark spanning both vision and language modalities. In total, our experimental setting covers \textbf{14 pre-trained backbones}, finetuned on \textbf{22 vision} and \textbf{11 language} tasks, allowing us to assess the robustness of our method across model scales and architectures. The complete list of models and datasets is presented in Appendix~\ref{app:additional_ft_setup_details}.

\paragraph{Quantization Setup.} We primarily study 3-bit weight-only quantization, which represents the regime in which vanilla PTQ is severely degraded while QAT can still recover substantial performance. Both PTQ and QAT target the weights of every \texttt{nn.Linear} PyTorch \citep{paszke2017automatic} module in the backbone,  excluding the classification head (\citep{torchao}). 
%
The biases, LayerNorm parameters, patch embeddings, and all non-linear modules are left intact. We employ symmetric channel-wise quantization but not activation quantization, activation smoothing, or rotation-based pre-processing \citep{xiao2023smoothquant, liu2024spinquant}. This design choice isolates the effect of QV transfer by preventing advanced pre-processing techniques from introducing confounding signals. 
\paragraph{Fine-Tuning Setup.} \label{sec:finetuning_setup} All weight-space arithmetic is performed between checkpoints sharing the same pretrained initialization. 
Each task is finetuned with and without QAT with identical optimizer, schedule, batch size, number of epochs, seed, and head initialization; the only difference is the insertion of fake-quantization modules during QAT. 
QVs are computed and transferred only on the shared backbone or encoder parameters, excluding task-specific heads. 
We evaluate this protocol on vision and text classifiers across both \textbf{learned-head} and \textbf{CLIP-style} \citep{radford2021learning} regimes. 
Complete dataset lists, model identifiers, and implementation details are provided in Appendix~\ref{app:additional_ft_setup_details}.

Please refer to Appendix~\ref{app:variants} for experiments on additional experimental settings.
\section{Results}
\label{sec:results}
\begin{figure}
    \centering
    \includegraphics[width=0.99\linewidth]{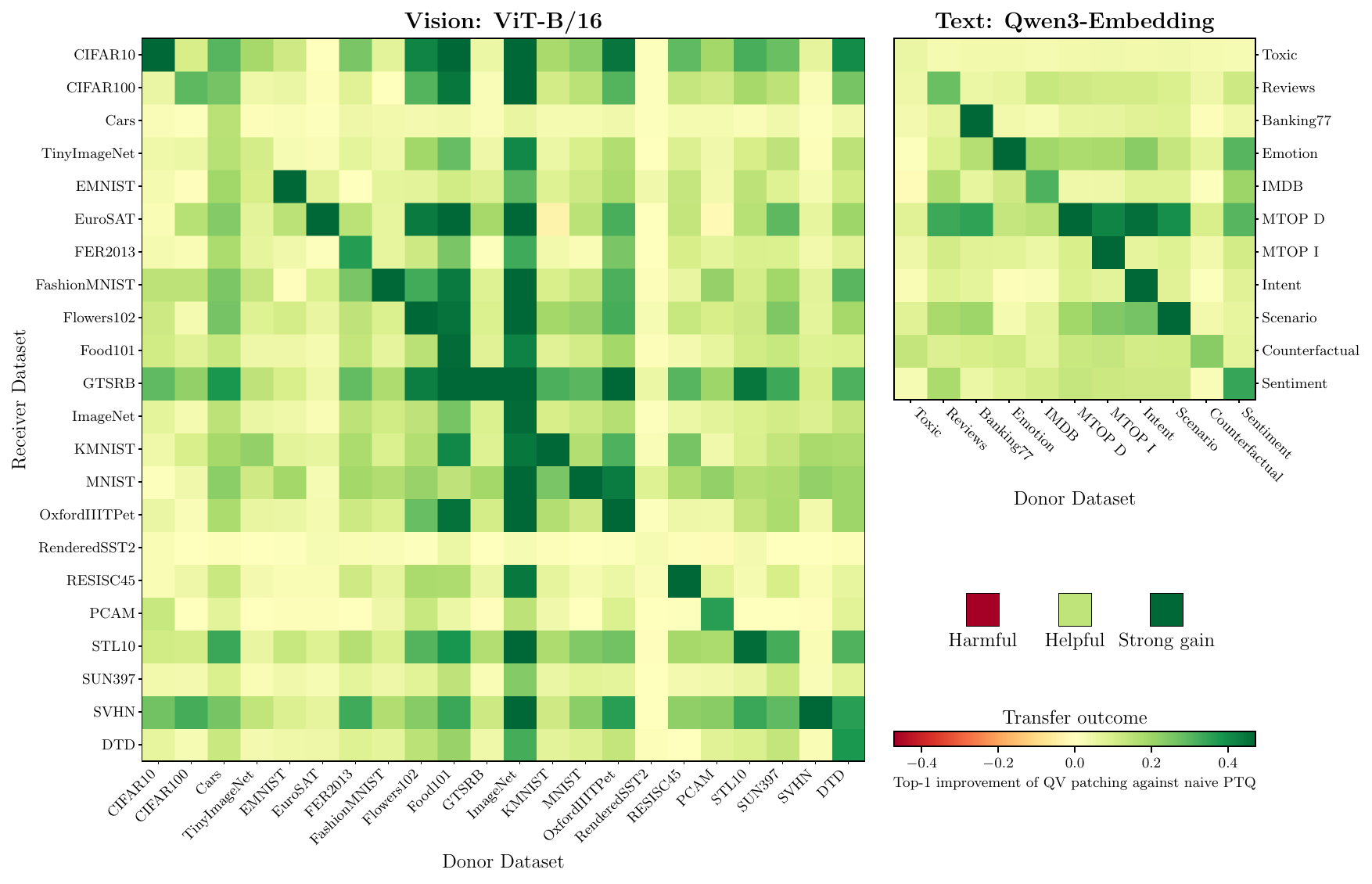}
    \caption{\textbf{Cross-task QV transfer for vision (left) and text (right) classifiers.} Top-1 accuracy change from patching receiver \(r\) with donor \(d\) quantization vector, relative to vanilla 3-bit PTQ. 
    }
    \label{fig:qv_transfer_heatmap_multimodal}
\end{figure}
%
%
\paragraph{Quantization Vector Direction.}
To isolate the effect of the QV direction from its magnitude, we first set the scaling factor \(\lambda=1\) from Eq.~\ref{eq:qv_patching}. 
Figures~\ref{fig:qv_transfer_heatmap_unit_vs_alpha_scale_timm_supervised_vit_b_16}, \ref{fig:qv_transfer_heatmap_unit_vs_alpha_scale_automodelforsequenceclassification_qwen3_embedding} provide distinct insights into transferability when respectively viewed row-wise or column-wise.
The rows show how a receiver dataset behaves when patched with available donors, revealing its receptiveness to robustness patching.
Conversely, the columns depict the effectiveness of a donor QV across all possible receivers, highlighting its generalizability and strength as a robustness injector.
Consequently, the intersection of row \(r\) and column \(d\) measures the net accuracy change when receiver \(r\) is patched with donor \(d\)'s QV, compared to the standard PTQ described in Eq.~\ref{eq:cross_task_transfer_metric}.

\paragraph{Quantization Vector Magnitude.}
Using a fixed-length vector across different tasks might cause the receiver model to either overshoot or undershoot the most useful region of weight space.
To investigate whether a suitable magnitude exists for each donor--receiver pair, we treat \(\lambda\) in Eq.~\ref{eq:qv_patching} as a scalar hyperparameter, similar to the use of task-vector coefficients in \citet{ilharco2022editing}.
This approach is consistent with Proposition~\ref{prop:alignment_controls_transfer}, which predicts that a donor direction can still be useful even when unit-scale transfer fails, as long as its optimal projection coefficient differs from 1.

In the updated protocol, \(\lambda\) is selected using only the receiver validation split.
For each donor--receiver pair, we sweep a fixed grid of candidate coefficients in $[0.15,\,1.5]$ on validation data with a step size of 0.15, choose the coefficient that maximizes validation Top-1 accuracy after 3-bit PTQ, and then report the corresponding performance on the receiver test split.
Formally, we select
$
\lambda^{\mathrm{val}}_{\mathcal D \to \mathcal R}
=
\arg\max_{\lambda \in \Lambda}
\mathrm{Acc}_{\mathrm{val}}
\left(
\mathrm{PTQ}(\theta_{\mathcal R}+\lambda\rho_{\mathcal D})
\right)$ and report test accuracy using this frozen value.
Seen through Proposition~\ref{prop:alignment_controls_transfer}, the validation
sweep estimates a practical projection coefficient for applying an already
useful donor direction to a given receiver.

The contrast with unscaled transfer is significant, as can be seen in Figures~\ref{fig:qv_transfer_heatmap_unit_vs_alpha_scale_timm_supervised_vit_b_16}, \ref{fig:qv_transfer_heatmap_unit_vs_alpha_scale_automodelforsequenceclassification_qwen3_embedding}.
First, validation-tuned scaling mitigates destructive interference.
The pronounced negative transfer areas observed under the \(\lambda=1\) protocol are substantially reduced when the magnitude is tuned.
In this setting, donor QVs rarely degrade the receiver's baseline PTQ
test-set performance. Second, validation-tuned scaling amplifies positive transfer.
For pairs that already exhibited gains at unit scale, calibrating \(\lambda\) yields additional improvements, suggesting that the QV often points toward a quantization-robust region, but the appropriate step length depends on the donor--receiver pair.
Overall, these results support a directional interpretation of QV transfer.
Negative transfer at \(\lambda=1\) is often due to an incorrect step size rather than a fundamentally incompatible donor direction.
At the same time, the calibrated setting is not fully data-free: unlike unit-scale patching, it uses a held-out validation split to select one scalar coefficient.
It remains receiver-training-free, since no receiver-side QAT or gradient-based
optimization is performed.

\paragraph{Highlights.} For each receiver task, we apply QV patching using the best donor task (excluding itself) and the best validation-tuned $\lambda$. For both vision and text domains, we find that QV patching achieves significant overall gains (see \textit{FT+PTQ} vs \textit{QV Patching + PTQ} in Figure \ref{fig:timm_supervised_automodelforsequenceclassification_qv_patching}), closely approaching the ideal, fully trained QAT+PTQ setting in vision and \textbf{even surpassing it on a few occasions}.
Because it selects both donor and $\lambda$ on receiver validation data, this is an upper bound on the data-free setting.
A cost analysis in Appendix~\ref{par:amortized_cost} shows that QV reuse can amortize the one-time donor QAT cost across multiple receivers, becoming cheaper than independent receiver-side QAT after sufficiently many deployments.

\paragraph{Bit widths and stronger quantizers.}
Our main experiments focus on 3-bit weight quantization, where vanilla PTQ incurs substantial but still recoverable degradation. 
At 4 bits, PTQ already remains close to full precision and therefore leaves little room for improvement, whereas at 2 bits even receiver-side QAT struggles to recover performance (Appendix~\ref{app:bit_widths}). 
%
QV-patched weights remain compatible with stronger PTQ methods such as AWQ \cite{lin2024awq} and GPTQ \cite{frantar2022gptq}; QV patching is not a competitor to them, and its benefit on top of them is small (Appendix~\ref{app:strong_ptq}).

\paragraph{Remark.} Because Theorem~\ref{thm:qv_well_defined} rules out generic gauge mismatch between matched FT and QAT endpoints in our regime, the results shown in Figure~\ref{fig:timm_supervised_automodelforsequenceclassification_qv_patching} should be interpreted as genuine  geometry transfer rather than as an artifact of subtracting incompatible checkpoints.

In Figure \ref{fig:qv_transfer_heatmap_multimodal}, we find that the performance gain across all donor-receiver pairs is non-negative if the proper $\lambda$ is applied, highlighting the generalizability of the method.
Interestingly, ImageNet emerges as an empirically near-universal donor on our benchmark, already at unit scale. We hypothesize that this stems from the breadth and diversity of tasks (over 20k classes), which induces activation and weight distributions with broader dynamic ranges and more representative quantization patterns. As a result, the corresponding QV captures outlier structures that are not overly specialized to a single downstream task, making it effective as a transferable correction signal. 

Please refer to Appendix~\ref{app:extended_results} for an extended analysis of results.

\section{Conclusions}
In this work, we investigated whether the robustness acquired during QAT can be isolated as a transferable property.
We introduced the quantization vector (QV), defined as the weight-space displacement between a standard fine-tuned checkpoint and its QAT-enabled counterpart.
Through extensive cross-task evaluations on vision and language models, we demonstrated that patching a standard model with a donor QV can substantially improve its robustness to 3-bit post-training quantization.

More broadly, our results suggest that QAT can induce a stable, partially transferable direction in weight space.
This expands the scope of weight-space arithmetic.
While previous work has largely focused on manipulating semantic capabilities, such as tasks, languages, or styles, our results show that weight-space directions can also encode structural and computational properties, like resilience to low-precision quantization.

The two theoretical results clarify why this picture is coherent.
Theorem~\ref{thm:qv_well_defined} shows that, in our regime, the QV is generically a genuine displacement in parameter space rather than an artifact of hidden symmetry mismatch, while Proposition~\ref{prop:alignment_controls_transfer} explains why successful transfer depends on donor--receiver alignment and on the patch scale.
Consistent with this view, rescaling the donor QV removes most destructive interference and reveals that much of the useful transfer signal lies in the direction itself.

Taken together, these results support a simple perspective: part of quantization robustness can be isolated, stored as a weight-space object, and reused across different tasks without the need to repeat full receiver-side QAT.
This suggests a broader understanding of low-bit adaptation, where robustness is not only trained, but can also be transferred.

\begin{ack}
This work is supported through the MUR FIS2 grant n. FIS-2023-00942 “NEXUS” (cup B53C25001030001).
This work was carried out while Daniele Solombrino was enrolled in the Italian National Doctorate on Artificial Intelligence run by Sapienza University of Rome.
We thank and acknowledge ISCRA for awarding this project access to the LEONARDO supercomputer, owned by the EuroHPC Joint Undertaking, hosted by CINECA (Italy).
We thank Fabrizio Fabbri for generously hosting and supporting part of the computing infrastructure used throughout this work.
The authors declare no competing interests.
\end{ack}

\bibliographystyle{plainnat}
\bibliography{bib, datasets}

\newpage
\null

\appendix

\section{Limitations}
\label{sec:limitations}
Our study represents a foundational step toward transferring quantization robustness, yet certain aspects require further exploration. 
Theorem~\ref{thm:qv_well_defined} addresses the well-defined nature of the displacement itself but does not solve the separate problem of selecting \(\lambda\) without held-out calibration data.

The validation sweep over \(\lambda\) suggests that a significant portion of the useful transfer signal lies in the QV direction, yet it does not establish universality. In practical deployment scenarios, a minimal held-out calibration set should be used rather than the test set.
Furthermore, to isolate the effect of the QV without introducing confounding variables, we intentionally employed a basic PTQ setup.
%
Patched weights remain fully compatible with stronger PTQ methods such as AWQ and GPTQ, although in this setting the benefits of QV patching are best obtained by calibrating the scale \(\lambda\) (Appendix~\ref{app:strong_ptq}).
Finally, since a QV is a displacement in parameter space, donor and receiver must share the same architecture and pretrained initialization; transferring QVs across architectures or initializations is left for future work.
%
We acknowledge these limitations and view them as opportunities for potential future work.

\section{Theorems and proofs}
%

\subsection{Proof of Proposition~\ref{prop:alignment_controls_transfer}}
\label{app:proof_cosine_squared_recovery}
The proof is exact line minimization of a positive-definite quadratic in the receiver's local geometry.
\begin{proof}
Because \(\theta_{\mathcal R}\) is fixed, it is enough to optimize over donor-induced
displacements of the form
\[
\delta(\lambda)=\lambda \rho_{\mathcal D}.
\]
Define the \(H_{\mathcal R}\)-inner product and norm by
\[
\langle u,v\rangle_{H_{\mathcal R}} := u^\top H_{\mathcal R} v,
\qquad
\|u\|_{H_{\mathcal R}}^2 := u^\top H_{\mathcal R} u.
\]
Under the local quadratic model,
\[
g_{\mathcal R}(\delta)
=
g_{\mathcal R}(\rho_{\mathcal R})
+
\frac12\|\delta-\rho_{\mathcal R}\|_{H_{\mathcal R}}^2,
\]
so along the donor line we have
\[
g_{\mathcal R}(\lambda\rho_{\mathcal D})
=
g_{\mathcal R}(\rho_{\mathcal R})
+
\frac12\|\lambda\rho_{\mathcal D}-\rho_{\mathcal R}\|_{H_{\mathcal R}}^2.
\]
Expanding the square gives
\[
g_{\mathcal R}(\lambda\rho_{\mathcal D})
=
g_{\mathcal R}(\rho_{\mathcal R})
+
\frac12\Bigl(
\lambda^2\|\rho_{\mathcal D}\|_{H_{\mathcal R}}^2
-2\lambda\langle \rho_{\mathcal D},\rho_{\mathcal R}\rangle_{H_{\mathcal R}}
+\|\rho_{\mathcal R}\|_{H_{\mathcal R}}^2
\Bigr).
\]
Since \(H_{\mathcal R}\succ 0\) and \(\rho_{\mathcal D}\neq 0\), this is a strictly convex quadratic function of \(\lambda\), so it has a unique minimizer. Differentiating with respect to \(\lambda\) and setting the derivative to zero yields
\[
\lambda^\star
=
\frac{\langle \rho_{\mathcal D},\rho_{\mathcal R}\rangle_{H_{\mathcal R}}}
{\|\rho_{\mathcal D}\|_{H_{\mathcal R}}^2}
=
\frac{\rho_{\mathcal D}^\top H_{\mathcal R}\rho_{\mathcal R}}
{\rho_{\mathcal D}^\top H_{\mathcal R}\rho_{\mathcal D}}.
\]

Substituting \(\lambda^\star\) back into the quadratic gives
\[
g_{\mathcal R}(\lambda^\star\rho_{\mathcal D})
=
g_{\mathcal R}(\rho_{\mathcal R})
+
\frac12\left(
\|\rho_{\mathcal R}\|_{H_{\mathcal R}}^2
-
\frac{\langle \rho_{\mathcal D},\rho_{\mathcal R}\rangle_{H_{\mathcal R}}^2}
{\|\rho_{\mathcal D}\|_{H_{\mathcal R}}^2}
\right).
\]
Now the receiver-side QAT gain relative to leaving the receiver unpatched is
\[
g_{\mathcal R}(0)-g_{\mathcal R}(\rho_{\mathcal R})
=
\frac12\|\rho_{\mathcal R}\|_{H_{\mathcal R}}^2,
\]
while the gain obtained by the best donor patch is
\[
g_{\mathcal R}(0)-g_{\mathcal R}(\lambda^\star\rho_{\mathcal D})
=
\frac12
\frac{\langle \rho_{\mathcal D},\rho_{\mathcal R}\rangle_{H_{\mathcal R}}^2}
{\|\rho_{\mathcal D}\|_{H_{\mathcal R}}^2}.
\]
Therefore, the fraction of receiver-side QAT gain recovered by the best donor patch is
\[
\frac{
g_{\mathcal R}(0)-g_{\mathcal R}(\lambda^\star\rho_{\mathcal D})
}{
g_{\mathcal R}(0)-g_{\mathcal R}(\rho_{\mathcal R})
}
=
\frac{\langle \rho_{\mathcal D},\rho_{\mathcal R}\rangle_{H_{\mathcal R}}^2}
{\|\rho_{\mathcal D}\|_{H_{\mathcal R}}^2\,\|\rho_{\mathcal R}\|_{H_{\mathcal R}}^2}
=
\cos_{H_{\mathcal R}}^2(\rho_{\mathcal D},\rho_{\mathcal R}).
\]

The three conclusions now follow immediately.

First, donor transfer helps if and only if the recovered fraction is strictly positive, which holds if and only if
\[
\langle \rho_{\mathcal D},\rho_{\mathcal R}\rangle_{H_{\mathcal R}} \neq 0,
\]
that is, if and only if \(\rho_{\mathcal D}\) is not \(H_{\mathcal R}\)-orthogonal to \(\rho_{\mathcal R}\).

Second, donor transfer matches receiver-side QAT if and only if the recovered fraction equals \(1\), i.e. if and only if
\[
\cos_{H_{\mathcal R}}^2(\rho_{\mathcal D},\rho_{\mathcal R})=1.
\]
By equality in the Cauchy--Schwarz inequality for the \(H_{\mathcal R}\)-inner product, this happens if and only if \(\rho_{\mathcal D}\) is a nonzero scalar multiple of \(\rho_{\mathcal R}\).

Third, donor transfer never exceeds receiver-side QAT because
\[
\cos_{H_{\mathcal R}}^2(\rho_{\mathcal D},\rho_{\mathcal R}) \le 1,
\]
again by Cauchy--Schwarz.
\end{proof}


\subsection{First deviation in local displacement}\label{sub:first_deviation_in_local_displacement}

\begin{proposition}[The cosine-squared law is second-order accurate]\label{prop:cosine_squared_second_order}
Let \(g_{\mathcal R}\) be twice continuously differentiable near \(\rho_{\mathcal R}\), assume
\[
\nabla g_{\mathcal R}(\rho_{\mathcal R}) = 0,
\qquad
H_{\mathcal R}:=\nabla^2 g_{\mathcal R}(\rho_{\mathcal R}) \succ 0,
\]
and let \(\lambda^\star\) be the donor scale from Proposition~\ref{prop:alignment_controls_transfer}. Suppose that \(\nabla^2 g_{\mathcal R}\) is \(L_{\mathcal R}\)-Lipschitz on a neighborhood containing the line segments from \(\rho_{\mathcal R}\) to \(0\) and to \(\lambda^\star \rho_{\mathcal D}\). Then
\[
g_{\mathcal R}(0)-g_{\mathcal R}(\lambda^\star \rho_{\mathcal D})
=
\cos_{H_{\mathcal R}}^2(\rho_{\mathcal D},\rho_{\mathcal R})
\bigl[g_{\mathcal R}(0)-g_{\mathcal R}(\rho_{\mathcal R})\bigr]
+
\varepsilon_{\mathcal R},
\]
with
\[
|\varepsilon_{\mathcal R}|
\le
\frac{L_{\mathcal R}}{6}
\left(
\|\rho_{\mathcal R}\|^3
+
\|\lambda^\star \rho_{\mathcal D}-\rho_{\mathcal R}\|^3
\right),
\]
where \(\|\cdot\|\) denotes the Euclidean norm.

In particular, the first deviation from Proposition~\ref{prop:alignment_controls_transfer} is cubic in the local displacement scale.
\end{proposition}

\begin{proof}
Write
\[
H := H_{\mathcal R},
\qquad
\rho := \rho_{\mathcal R},
\qquad
\eta := \rho_{\mathcal D},
\qquad
\delta^\star := \lambda^\star \eta.
\]
Since \(\nabla g_{\mathcal R}(\rho)=0\), Taylor's theorem with integral remainder around \(\rho\) gives, for every \(\delta\) on the relevant line segments,
\[
g_{\mathcal R}(\delta)
=
g_{\mathcal R}(\rho)
+
\frac12(\delta-\rho)^\top H(\delta-\rho)
+
r(\delta),
\]
where
\[
r(\delta)
=
\int_0^1
(1-t)\,
(\delta-\rho)^\top
\bigl(\nabla^2 g_{\mathcal R}(\rho+t(\delta-\rho))-H\bigr)
(\delta-\rho)\,dt.
\]
Because \(\nabla^2 g_{\mathcal R}\) is \(L_{\mathcal R}\)-Lipschitz,
\[
\bigl\|
\nabla^2 g_{\mathcal R}(\rho+t(\delta-\rho)) - H
\bigr\|
\le
L_{\mathcal R}\, t\, \|\delta-\rho\|.
\]
Therefore
\[
|r(\delta)|
\le
\int_0^1
(1-t)\,L_{\mathcal R}\,t\,\|\delta-\rho\|^3\,dt
=
\frac{L_{\mathcal R}}{6}\,\|\delta-\rho\|^3.
\]

Applying this expansion at \(\delta=0\) and \(\delta=\delta^\star\) yields
\[
g_{\mathcal R}(0)-g_{\mathcal R}(\delta^\star)
=
\frac12\|\rho\|_H^2
-
\frac12\|\delta^\star-\rho\|_H^2
+
r(0)-r(\delta^\star),
\]
where \(\|u\|_H^2 := u^\top H u\).

Now let
\[
c := \cos_H^2(\eta,\rho).
\]
By Proposition~\ref{prop:alignment_controls_transfer}, the quadratic part satisfies
\[
\frac12\|\rho\|_H^2
-
\frac12\|\delta^\star-\rho\|_H^2
=
c\,\frac12\|\rho\|_H^2.
\]
Also,
\[
g_{\mathcal R}(0)-g_{\mathcal R}(\rho)
=
\frac12\|\rho\|_H^2 + r(0).
\]
Substituting this into the previous identity gives
\[
g_{\mathcal R}(0)-g_{\mathcal R}(\delta^\star)
=
c\,[g_{\mathcal R}(0)-g_{\mathcal R}(\rho)]
+
\varepsilon_{\mathcal R},
\]
where
\[
\varepsilon_{\mathcal R}
=
(1-c)\,r(0)-r(\delta^\star).
\]

Since \(0\le c\le 1\),
\[
|\varepsilon_{\mathcal R}|
\le
|r(0)| + |r(\delta^\star)|.
\]
Using the remainder bound proved above at \(\delta=0\) and \(\delta=\delta^\star\), we obtain
\[
|\varepsilon_{\mathcal R}|
\le
\frac{L_{\mathcal R}}{6}
\left(
\|\rho\|^3 + \|\delta^\star-\rho\|^3
\right),
\]
which is exactly the claimed bound.
\end{proof}


\subsection{Quantization vectors are well-defined}
\label{app:wd}

Throughout this subsection, \(g^{\odot 2}\) denotes elementwise squaring, and all divisions are coordinatewise.

\begin{theorem}[AdamW kills continuous linear reparameterizations]
\label{thm:adamw_signed_permutation}
Consider an AdamW update of the form
\[
m^+ = \beta_1 m + (1-\beta_1) g,
\qquad
v^+ = \beta_2 v + (1-\beta_2) g^{\odot 2},
\qquad
\theta^+ = a_t \theta - b_t \frac{m^+}{\sqrt{v^+} + \varepsilon \mathbf 1},
\]
where \(a_t > 0\), \(b_t > 0\), \(\beta_1,\beta_2 \in (0,1)\), and \(\varepsilon > 0\).

Suppose an invertible linear map \(S \in GL_n(\mathbb R)\) on parameter coordinates is an exact linear equivariance of this update, in the sense that there exists a linear map \(R\) on second-moment coordinates such that, for all \(\theta,m,v,g\),
\[
U_t\!\left(S\theta,\; S^{-\top} m,\; Rv;\; S^{-\top} g\right)
=
\left(S\theta^+,\; S^{-\top} m^+,\; Rv^+\right).
\]
Then \(S\) is a signed permutation matrix. In particular, AdamW cannot admit a nontrivial continuous family of exact linear reparameterization symmetries.
\end{theorem}

\begin{proof}
Let \(T := S^{-\top}\). Equivariance of the second-moment update gives
\[
R\!\left(g^{\odot 2}\right) = (Tg)^{\odot 2}
\qquad \text{for all } g \in \mathbb R^n.
\]
We first show that every row of \(T\) has at most one nonzero entry. If row \(i\) had two nonzero entries \(a = T_{ij}\) and \(b = T_{ik}\) with \(j \neq k\), then for
\[
g = t e_j + e_k
\]
the \(i\)-th coordinate of the right-hand side would be
\[
\bigl((Tg)^{\odot 2}\bigr)_i = (a t + b)^2 = a^2 t^2 + 2ab\, t + b^2,
\]
while the left-hand side would be
\[
\bigl(R(g^{\odot 2})\bigr)_i
=
\bigl(R(t^2 e_j + e_k)\bigr)_i,
\]
which is affine in \(t^2\) and therefore contains no linear term in \(t\). Since \(ab \neq 0\), this is impossible. Hence, each row of \(T\) has at most one nonzero entry.

Because \(T\) is invertible, each column must also have exactly one nonzero entry. Therefore
\[
T = DP
\]
for some permutation matrix \(P\) and some diagonal matrix
\[
D = \mathrm{diag}(d_1,\dots,d_n)
\]
with each \(d_i \neq 0\). Since every vector in \(\mathbb R_{\ge 0}^n\) can be written as \(g^{\odot 2}\) for some \(g\), the identity
\[
R\!\left(g^{\odot 2}\right) = (Tg)^{\odot 2}
\]
implies
\[
R = D^2 P.
\]

Now use equivariance of the parameter update. Since the scalar factor \(a_t\) commutes with every linear map, the adaptive term must satisfy
\[
S\!\left(\frac{m}{\sqrt{v}+\varepsilon \mathbf 1}\right)
=
\frac{Tm}{\sqrt{Rv}+\varepsilon \mathbf 1}
\qquad \text{for all } m,v.
\]
Choose
\[
m = P^{-1} e_i,
\qquad
v = t\, P^{-1} e_i
\qquad (t \ge 0).
\]
Then
\[
Tm = d_i e_i,
\qquad
Rv = d_i^2 t\, e_i,
\qquad
S = T^{-\top} = D^{-1}P.
\]
Therefore the \(i\)-th coordinate of the previous identity becomes
\[
\frac{d_i^{-1}}{\sqrt{t}+\varepsilon}
=
\frac{d_i}{|d_i| \sqrt{t}+\varepsilon}
\qquad \text{for all } t \ge 0.
\]
Comparing the coefficients of \(\sqrt{t}\) and the constant terms forces
\[
|d_i| = 1.
\]
Since this holds for every \(i\), \(D\) has only \(\pm 1\) on the diagonal. Hence \(S\) is a signed permutation matrix. In particular, no nontrivial continuous linear family can satisfy the equivariance identity.
\end{proof}

\begin{lemma}[The fake quantizer preserves the remaining finite gauge]
\label{lem:fq_commutes}
Let \(\mathcal G\) be the finite permutation gauge left by the backbone parameterization. For every quantized linear weight tensor \(W\) and every \(g \in \mathcal G\),
\[
FQ(g \!\cdot\! W) = g \!\cdot\! FQ(W).
\]
Consequently, the STE forward rule used in QAT,
\[
W \longmapsto W + \bigl(FQ(W)-W\bigr)_{\mathrm{stopgrad}},
\]
commutes with the same action of \(g\).
\end{lemma}

\begin{proof}
The operator \(FQ\) used in this paper is built from row-wise absolute maxima, division by the resulting scales, element-wise rounding, element-wise clipping to a symmetric signed grid, and rescaling. Row and column permutations merely reorder entries and, in the row-wise case, the corresponding scales. Because the same symmetric quantization rule is then applied coordinate-wise, each of these operations commutes with the permutation actions in \(\mathcal G\). Hence
\[
FQ(g \!\cdot\! W) = g \!\cdot\! FQ(W).
\]
The STE rule is affine in \(W\) and \(FQ(W)\), so it commutes with the same action as well.
\end{proof}

\begin{proof}[\textbf{Proof of Theorem}~\ref{thm:qv_well_defined}]
The proof follows a standard real-analytic zero-set argument; see also \citet{nikolaou2026language} for a recent use of the same measure-zero template in a transformer setting.

Write the piecewise real-analytic assumption using a finite common partition of parameter space into full-dimensional cells on each of which both endpoint maps \(\Phi_\mathcal{D}^{\mathrm{FT}}\) and \(\Phi_\mathcal{D}^{\mathrm{QAT}}\) are real-analytic.

By Theorem~\ref{thm:adamw_signed_permutation}, matched AdamW fine-tuning cannot preserve any nontrivial continuous linear gauge. In the ViT regime considered here, the only exact symmetries that remain are the finite hidden-unit and attention-head relabelings collected in \(\mathcal G\). By Lemma~\ref{lem:fq_commutes}, the QAT fake quantizer commutes with that same finite permutation gauge, so the FT and QAT endpoint maps live in a common discrete-gauge setting.

Fix \(g \in \mathcal G \setminus \{e\}\). On each full-dimensional cell \(C\), define
\[
h_g(\theta)
:=
\left\|
\Phi_\mathcal{D}^{\mathrm{QAT}}(\theta) - g\,\Phi_\mathcal{D}^{\mathrm{FT}}(\theta)
\right\|^2.
\]
By construction, \(h_g\) is real-analytic on \(C\). By hypothesis, it is not identically zero on any such cell. Therefore its zero set on \(C\),
\[
Z_{g,C}
:=
\left\{
\theta \in C :
\Phi_\mathcal{D}^{\mathrm{QAT}}(\theta) = g\,\Phi_\mathcal{D}^{\mathrm{FT}}(\theta)
\right\},
\]
has Lebesgue measure zero.

Now consider the bad set
\[
\mathcal B_D
:=
\bigcup_{g \in \mathcal G \setminus \{e\}}
\left\{
\theta :
\Phi_\mathcal{D}^{\mathrm{QAT}}(\theta) = g\,\Phi_\mathcal{D}^{\mathrm{FT}}(\theta)
\right\}.
\]
This set is contained in the union of all cell boundaries and all zero sets \(Z_{g, C}\). The cell boundaries have measure zero, and each \(Z_{g,C}\) has measure zero. Since there are only finitely many cells and finitely many \(g \in \mathcal G\), \(\mathcal B_\mathcal{D}\) itself has measure zero.

Therefore, for every common initialization \(\theta_{\mathrm{pre}} \notin \mathcal B_\mathcal{D}\), the matched endpoints
\[
\theta_\mathcal{D} = \Phi_\mathcal{D}^{\mathrm{FT}}(\theta_{\mathrm{pre}}),
\qquad
\theta_{\mathcal{D},\mathrm{QAT}} = \Phi_\mathcal{D}^{\mathrm{QAT}}(\theta_{\mathrm{pre}})
\]
lie in the same gauge, and the difference
\[
\rho_\mathcal{D} = \theta_{\mathcal{D},\mathrm{QAT}} - \theta_\mathcal{D}
\]
is a well-defined parameter-space displacement.
\end{proof}


\section{Additional Fine-Tuning Setup Details}
\label{app:additional_ft_setup_details}
\begin{wraptable}{r}{6.2cm}
\vspace{-0.45cm}
\centering
\caption{\textbf{Vision datasets.} Dataset names, citations, and Hugging Face identifiers used in the vision experiments.}
\vspace{-0.2cm}
\label{tab:vision_dataset_hf_ids}
\small
\resizebox{0.45\textwidth}{!}{
\begin{tabular}{ll}
\toprule
\textbf{Dataset} & \textbf{Hugging Face identifier} \\
\midrule
Stanford Cars~\citep{krause_3d_2013} & \texttt{tanganke/stanford\_cars} \\
CIFAR-10~\citep{krizhevsky_learning_nodate} & \texttt{uoft-cs/cifar10} \\
CIFAR-100~\citep{krizhevsky_learning_nodate} & \texttt{uoft-cs/cifar100} \\
DTD~\citep{cimpoi_describing_2014} & \texttt{tanganke/dtd} \\
EMNIST~\citep{cohen_emnist_2017} & \texttt{tanganke/emnist\_letters} \\
EuroSAT~\citep{helber_eurosat_2019} & \texttt{tanganke/eurosat} \\
Fashion-MNIST~\citep{xiao_fashion-mnist_2017} & \texttt{zalando-datasets/fashion\_mnist} \\
FER2013~\citep{goodfellow_challenges_2013} & \texttt{clip-benchmark/wds\_fer2013} \\
Flowers102~\citep{nilsback_automated_2008} & \texttt{dpdl-benchmark/oxford\_flowers102} \\
Food-101~\citep{bossard_food-101_2014} & \texttt{ethz/food101} \\
GTSRB~\citep{stallkamp_german_2011} & \texttt{tanganke/gtsrb} \\
ImageNet~\citep{ILSVRC15} & \texttt{ILSVRC/imagenet-1k} \\
KMNIST~\citep{clanuwat_deep_2018} & \texttt{tanganke/kmnist} \\
MNIST~\citep{726791} & \texttt{ylecun/mnist} \\
Oxford-IIIT Pet~\citep{parkhi_cats_2012} & \texttt{timm/oxford-iiit-pet} \\
PCam~\citep{veeling_rotation_2018} & \texttt{1aurent/PatchCamelyon} \\
RESISC45~\citep{cheng_remote_2017} & \texttt{tanganke/resisc45} \\
Rendered SST-2~\citep{socher_recursive_nodate} & \texttt{nateraw/rendered-sst2} \\
STL-10~\citep{coates_analysis_2011} & \texttt{tanganke/stl10} \\
SUN397~\citep{xiao_sun_2016} & \texttt{tanganke/sun397} \\
SVHN~\citep{netzer_reading_nodate} & \texttt{ufldl-stanford/svhn} \\
Tiny ImageNet~\citep{le2015tiny} & \texttt{zh-plus/tiny-imagenet} \\
\bottomrule
\end{tabular}
}
\vspace{-0.55cm}
\end{wraptable}
Following \citet{ilharco2022editing}, all weight-space arithmetic is performed between checkpoints obtained from a shared pretrained initialization. 
For each task and architecture, we train matched full-precision and QAT checkpoints with identical settings, including optimizer, learning-rate schedule, batch size, number of epochs, total number of training steps, and random seed. 
The only difference between the two runs is that QAT inserts fake-quantization modules into the targeted linear layers during training. 
Additionally, when a task-specific classification head is trained, both the full-precision and QAT models share the same randomly initialized head.
We evaluate this protocol in both vision and text classification tasks. 
A complete list of datasets, along with their canonical citations and Hugging Face dataset identifiers, is provided here.

\paragraph{Vision Classifiers.}
For image classification, we construct classifiers by pairing a pretrained visual backbone with a task-specific classification head. 
We evaluate two head regimes. 
In the \emph{learned-head} regime, the classifier consists of a task-specific linear head initialized from a common random seed and trained jointly with the backbone. 
In the \emph{CLIP-style} regime \citep{radford2021learning}, the classifier uses a frozen bank of label-text embeddings to compute image-text logits, with only the visual backbone being fine-tuned.
For Vision Transformers \citep{dosovitskiy2020image}, we utilize \texttt{\texttt{ViT-B/16}}, \texttt{ViT-L/14}, and \texttt{ViT-H/14} under both the learned-head and CLIP-style regimes. 
In addition, for the learned-head regime, we also evaluate \texttt{Swin-B} and \texttt{Swin-L} \citep{liu2021swin}, as well as \texttt{DeiT-III-B} and \texttt{DeiT-III-L} \citep{touvron2022deit}. 
All vision models are fine-tuned at \(224\times224\) resolution.

\paragraph{Text Classifiers.}
We also apply the same fine-tuning protocol to text classification using pretrained text encoders. 
Specifically, we fine-tune \texttt{BERT-base}, \texttt{BERT-large} \citep{devlin2019bert}, \texttt{EmbeddingGemma-300M} \citep{vera2025embeddinggemma}, and \texttt{\texttt{Qwen3-Embedding}-0.6B} \citep{zhang2025qwen3} as the encoder backbones, each paired with a task-specific classification head.

In both vision and text experiments, quantization-vector subtraction and donor patching are applied only to the shared backbone or encoder parameters. 
Task-specific heads are never included in the quantization vector. 
Therefore, for a learned-head receiver, patching modifies the backbone or encoder while keeping the learned head fixed; for a CLIP-style receiver, patching modifies the visual backbone while keeping the frozen text classifier unchanged.
Following the work of \citet{ilharco2022editing} and \citet{gargiulo2025task}, arithmetic operations are performed only between checkpoints that share compatible parameterizations: the same architecture, the same pretrained initialization, and the same head regime. 
This compatibility is important because the quantization vector represents a coordinate-wise displacement in the shared backbone or encoder parameter space.

Table~\ref{tab:vision_dataset_hf_ids} presents the vision classification datasets utilized in the experiments, along with their respective Hugging Face identifiers. 
For the SVHN dataset, the experiments use the \texttt{cropped\_digits} configuration from \texttt{ufldl-stanford/svhn}.
Additionally, Table~\ref{tab:text_dataset_hf_ids} outlines the text classification datasets from MTEB \citep{muennighoff2023mteb, enevoldsenmmteb} and their Hugging Face identifiers.

Table~\ref{tab:vision_model_identifiers} provides a detailed list of the exact vision backbone identifiers, frameworks, head configurations, and pretrained versions utilized in the experiments. 
For CLIP-style vision experiments, models are loaded with the \texttt{open\_clip} framework~\citep{ilharco_gabriel_2021_5143773}. 
In contrast, for learned-head vision experiments, models employ the \texttt{timm} framework~\citep{rw2019timm}. 
Table~\ref{tab:text_model_identifiers} outlines the text encoders used for learned-head text classification, which are loaded with the \texttt{transformers} library~\citep{wolf-etal-2020-transformers}.
\begin{table}[]
\small
\centering
\caption{\textbf{Text datasets.} MTEB classification datasets, citations, and Hugging Face identifiers used in the text experiments.}
\label{tab:text_dataset_hf_ids}
\resizebox{\textwidth}{!}{
\begin{tabular}{ll}
\toprule
\textbf{Dataset} & \textbf{Hugging Face identifier} \\
\midrule
Emotion~\citep{saravia-etal-2018-carer} & \texttt{mteb/emotion} \\
IMDB~\citep{maas-etal-2011-learning} & \texttt{mteb/imdb} \\
Banking77~\citep{casanueva-etal-2020-efficient} & \texttt{mteb/banking77} \\
AmazonReviewsClassification~\citep{keung2020multilingual} & \texttt{mteb/AmazonReviewsClassification} \\
AmazonCounterfactualClassification~\citep{oneill-etal-2021-wish} & \texttt{mteb/amazon\_counterfactual} \\
MassiveIntentClassification~\citep{fitzgerald2022massive} & \texttt{mteb/amazon\_massive\_intent} \\
MassiveScenarioClassification~\citep{fitzgerald2022massive} & \texttt{mteb/amazon\_massive\_scenario} \\
MTOPDomainClassification~\citep{li-etal-2021-mtop} & \texttt{mteb/mtop\_domain} \\
MTOPIntentClassification~\citep{li-etal-2021-mtop} & \texttt{mteb/mtop\_intent} \\
ToxicConversationsClassification~\citep{jigsaw-unintended-bias-in-toxicity-classification} & \texttt{mteb/toxic\_conversations\_50k} \\
TweetSentimentExtractionClassification~\citep{tweet-sentiment-extraction} & \texttt{mteb/tweet\_sentiment\_extraction} \\
\bottomrule
\end{tabular}
}
\end{table}
\begin{table}[]
\centering
\caption{\textbf{Vision model identifiers.} Exact vision backbone identifiers, head regimes, frameworks, and pretrained versions used in the experiments.}
\label{tab:vision_model_identifiers}
\small
\resizebox{\textwidth}{!}{
\begin{tabular}{llll}
\toprule
\textbf{Head regime} & \textbf{Framework} & \textbf{Model identifier} & \textbf{Pretrained version} \\
\midrule
CLIP-style & \texttt{open\_clip} & \texttt{\texttt{ViT-B/16}} & \texttt{laion2b\_s34b\_b88k} \\
CLIP-style & \texttt{open\_clip} & \texttt{\texttt{ViT-L/14}} & \texttt{laion2b\_s32b\_b82k} \\
CLIP-style & \texttt{open\_clip} & \texttt{\texttt{ViT-H/14}} & \texttt{laion2b\_s32b\_b79k} \\
\midrule
Learned head & \texttt{timm} & \texttt{deit3\_base\_patch16\_224} & \texttt{fb\_in1k} \\
Learned head & \texttt{timm} & \texttt{deit3\_large\_patch16\_224} & \texttt{fb\_in1k} \\
Learned head & \texttt{timm} & \texttt{swin\_base\_patch4\_window7\_224} & \texttt{ms\_in22k\_ft\_in1k} \\
Learned head & \texttt{timm} & \texttt{swin\_large\_patch4\_window7\_224} & \texttt{ms\_in22k\_ft\_in1k} \\
Learned head & \texttt{timm} & \texttt{vit\_base\_patch16\_224} & \texttt{orig\_in21k} \\
Learned head & \texttt{timm} & \texttt{vit\_large\_patch16\_224} & \texttt{orig\_in21k} \\
Learned head & \texttt{timm} & \texttt{vit\_huge\_patch14\_224} & \texttt{orig\_in21k} \\
\bottomrule
\end{tabular}
}
\end{table}
\begin{table}[]
\centering
\caption{\textbf{Text encoder identifiers.} Exact Hugging Face identifiers used for learned-head text classification.}
\label{tab:text_model_identifiers}
\small
\begin{tabular}{ll}
\toprule
\textbf{Model} & \textbf{Hugging Face identifier} \\
\midrule
\texttt{BERT-base} & \texttt{google-bert/\texttt{BERT-base}-uncased} \\
\texttt{BERT-large} & \texttt{google-bert/\texttt{BERT-large}-uncased} \\
\texttt{\texttt{EmbeddingGemma}-300M} & \texttt{google/\texttt{\texttt{EmbeddingGemma}-300M}} \\
\texttt{\texttt{Qwen3-Embedding}-0.6B} & \texttt{Qwen/\texttt{\texttt{Qwen3-Embedding}-0.6B}} \\
\bottomrule
\end{tabular}
\end{table}
\section{Compute Setup}
\label{app:compute_setup}

Fine-tunings and evaluations were carried out using a mix of local and cloud machines. In all cases, single-GPU setups were used.
The local machines were equipped with 64\,GB of RAM, 8- or 16-core CPUs, and one of the following GPUs: an NVIDIA RTX 4090, RTX 3090, or RTX 3090\,Ti. 
The cloud machines were equipped with 128\,GB of RAM, 8 CPU cores, and an NVIDIA A100 (64\,GB).
Unless otherwise specified, all experiments were carried out using full precision (float32). 
The total compute budget was approximately 20{,}000 GPU-hours.
\section{Baselines}
\label{app:baselines}

\begin{figure}
    \centering
    \includegraphics[width=0.99\linewidth]{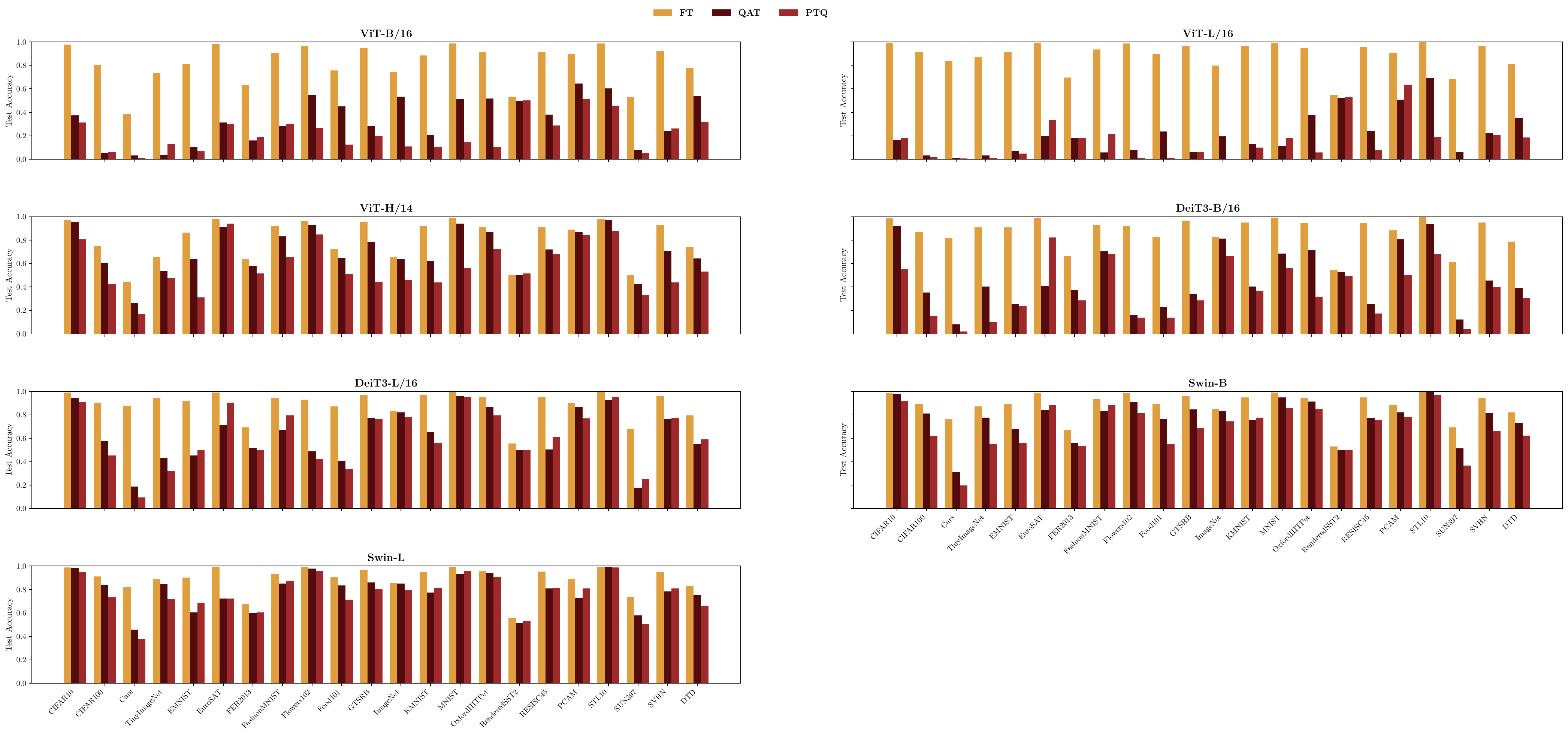}
    \caption{Comparison between accuracy of the full precision finetuned model (FT), the corresponding QAT and FP + PTQ, for the vision backbones used throughout the paper.}
    \label{fig:baselines_timm_supervised_baseline_bar_all_models}
\end{figure}
\begin{figure}
    \centering
    \includegraphics[width=0.99\linewidth]{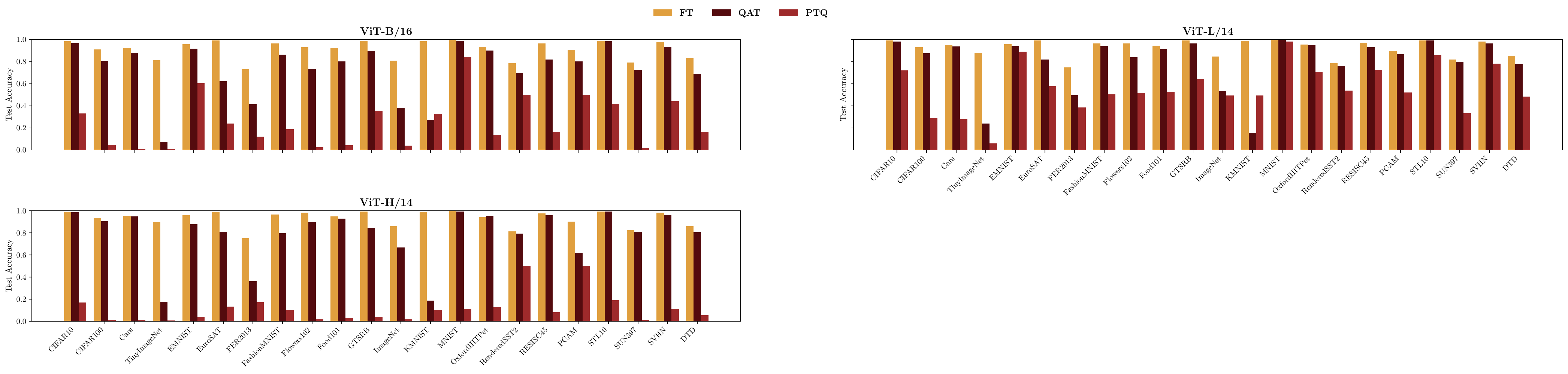}
    \caption{Comparison between accuracy of the full precision finetuned model (FT), the corresponding QAT and FP + PTQ, for the clip-style vision backbones used throughout the paper.}
    \label{fig:baselines_open_clip_baseline_bar_all_models}
\end{figure}
\begin{figure}
    \centering
    \includegraphics[width=0.99\linewidth]{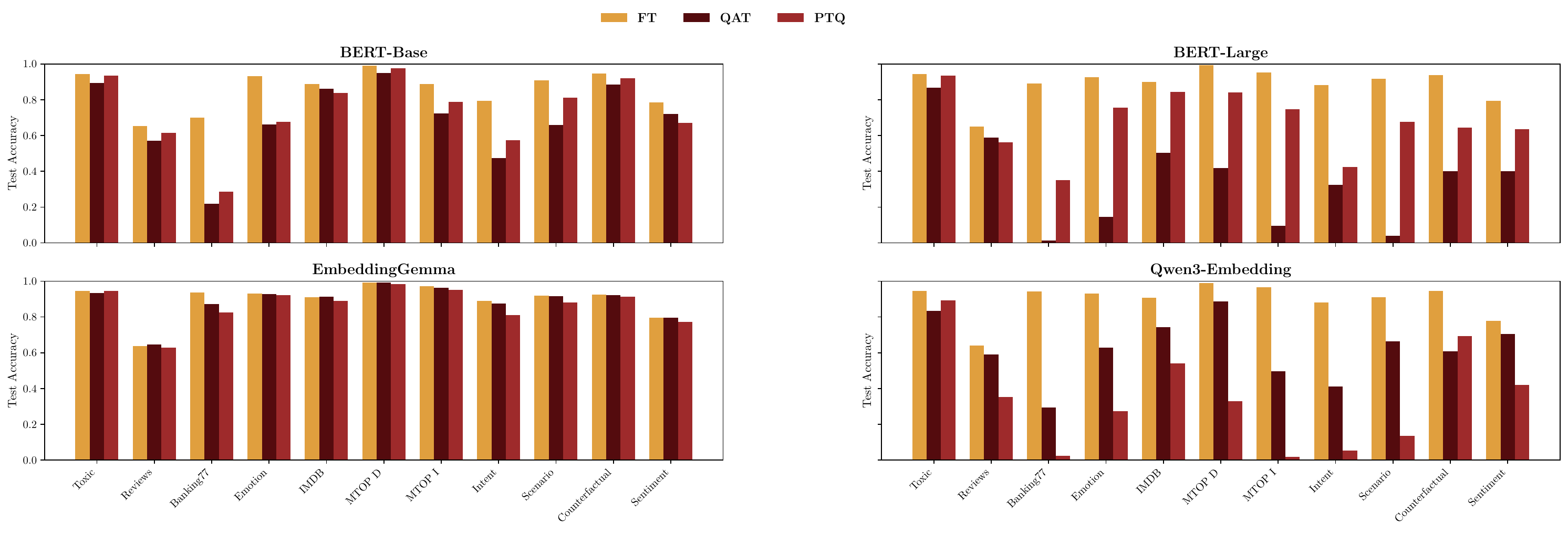}
    \caption{Comparison between accuracy of the full precision finetuned model, the corresponding QAT and FP + PTQ, for the text backbones used throughout the paper.}
    \label{fig:automodelforsequenceclassification_baseline_bar_all_models}
\end{figure}

Figures \ref{fig:baselines_timm_supervised_baseline_bar_all_models}, \ref{fig:baselines_open_clip_baseline_bar_all_models} and \ref{fig:automodelforsequenceclassification_baseline_bar_all_models} report the absolute Top-1 accuracies for all the models, including \texttt{\texttt{ViT-B/16}}, ViT-L/16, and \texttt{ViT-H/14} under both the learned-head and CLIP-style regimes, as well as for the learned-head regime of \texttt{Swin-B}, \texttt{Swin-L}, \texttt{DeiT-III-B}, and \texttt{DeiT-III-L} for vision and \texttt{BERT-base}, \texttt{BERT-large}, \texttt{\texttt{EmbeddingGemma}}, and \texttt{Qwen3-Embedding} for text, used under three different conditions: standard fine-tuning (FT), PTQ to 3-bit, and 3-bit QAT.
These baseline results serve two main purposes. First, they quantify the impact of the low-bit quantization regime addressed in this study. 
As shown by the images, across most datasets, PTQ induces a significant drop relative to full-precision (FT), confirming that 3-bit weight-only quantization is a challenging scenario in which robustness to quantization noise is not guaranteed. 
Second, the QAT results indicate that much of the performance lost due to PTQ can be regained when the model is explicitly trained to mitigate quantization effects. 
This establishes a notable gap between the results from PTQ and QAT, which our Quantization Vector aims to partially bridge through zero-shot transfer.

Additionally, the images also highlight that the quantization difficulty is highly task-dependent. 
Some datasets demonstrate relative stability under PTQ, while others suffer severe performance degradation, to the point of collapsing almost entirely, even though they perform strongly at FP. 
This variability motivates our donor-receiver analysis: if QAT-induced robustness can be transferred across tasks, then tasks that are particularly susceptible to PTQ may benefit from robustness directions extracted from more transferable donor tasks.

\section{Extended Results} 
\label{app:extended_results}
\begin{figure}[]
    \centering
    \includegraphics[width=0.99\linewidth]{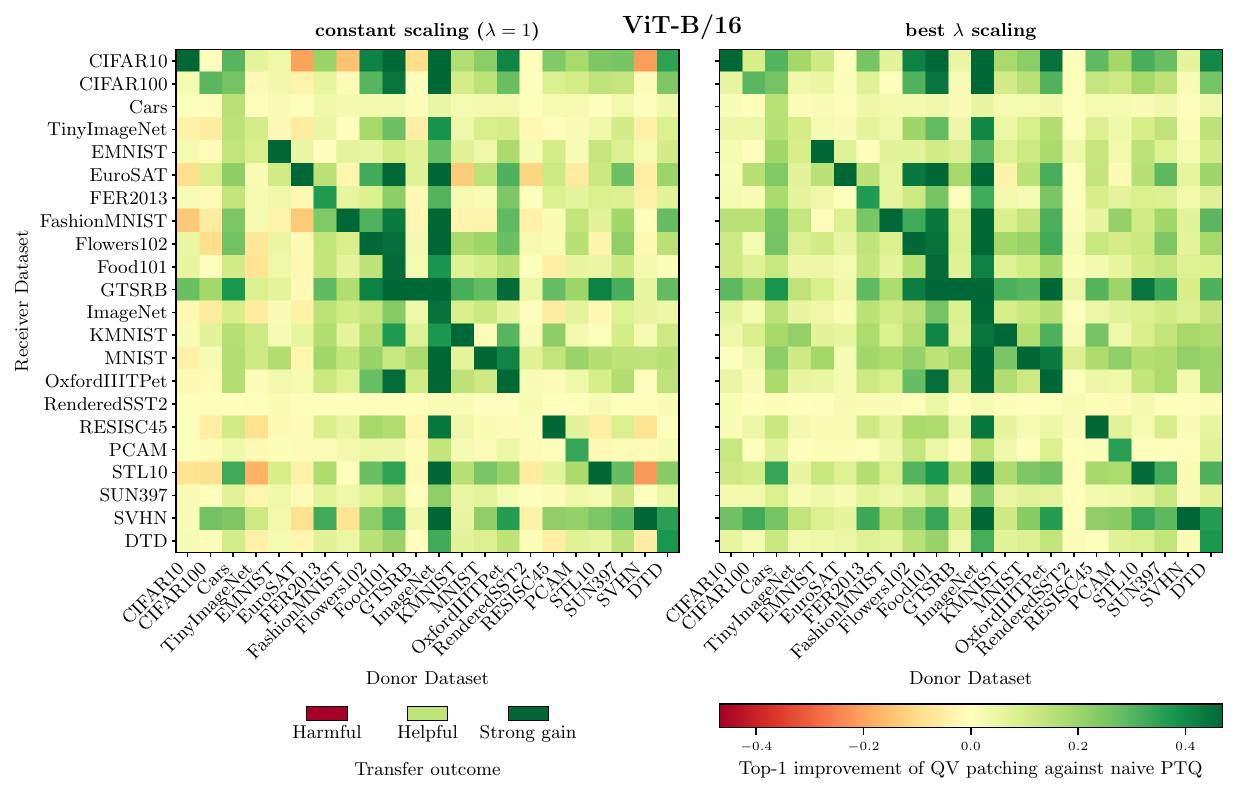}
    \caption{\textbf{QV transferability for learned-head \texttt{\texttt{ViT-B/16}} (vision)}. Top-1 accuracy change ($\Delta$) from patching receiver $r$ with donor $d$ quantization vector, relative to baseline 3-bit PTQ. For each architecture, heatmaps on the left show transfer using a constant scaling factor ($\lambda=1$), revealing both positive transfer and destructive interference. Heatmaps on the right demonstrate that modulating the magnitude $\lambda$ eliminates destructive interference while maximizing gains.}
\label{fig:qv_transfer_heatmap_unit_vs_alpha_scale_timm_supervised_vit_b_16}
\end{figure}
\begin{figure}[]
    \centering
    \includegraphics[width=0.99\linewidth]{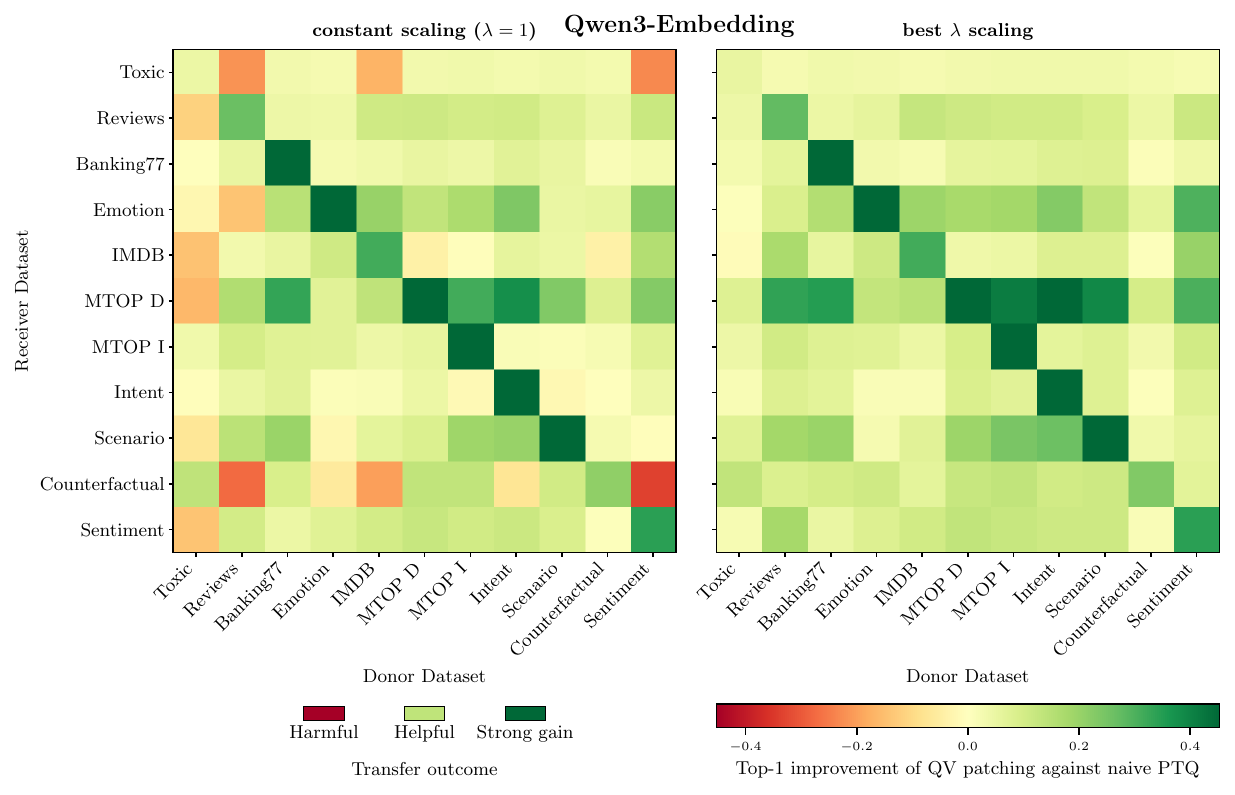}
    \caption{\textbf{QV transferability for \texttt{Qwen3-Embedding} (text)}. Top-1 accuracy change ($\Delta$) from patching receiver $r$ with donor $d$ quantization vector, relative to baseline 3-bit PTQ. For each architecture, heatmaps on the left show transfer using a constant scaling factor ($\lambda=1$), revealing both positive transfer and destructive interference. Heatmaps on the right demonstrate that modulating the magnitude $\lambda$ eliminates destructive interference while maximizing gains.}
\label{fig:qv_transfer_heatmap_unit_vs_alpha_scale_automodelforsequenceclassification_qwen3_embedding}
\end{figure}
\begin{figure}[]
    \centering
    \includegraphics[width=0.99\linewidth]{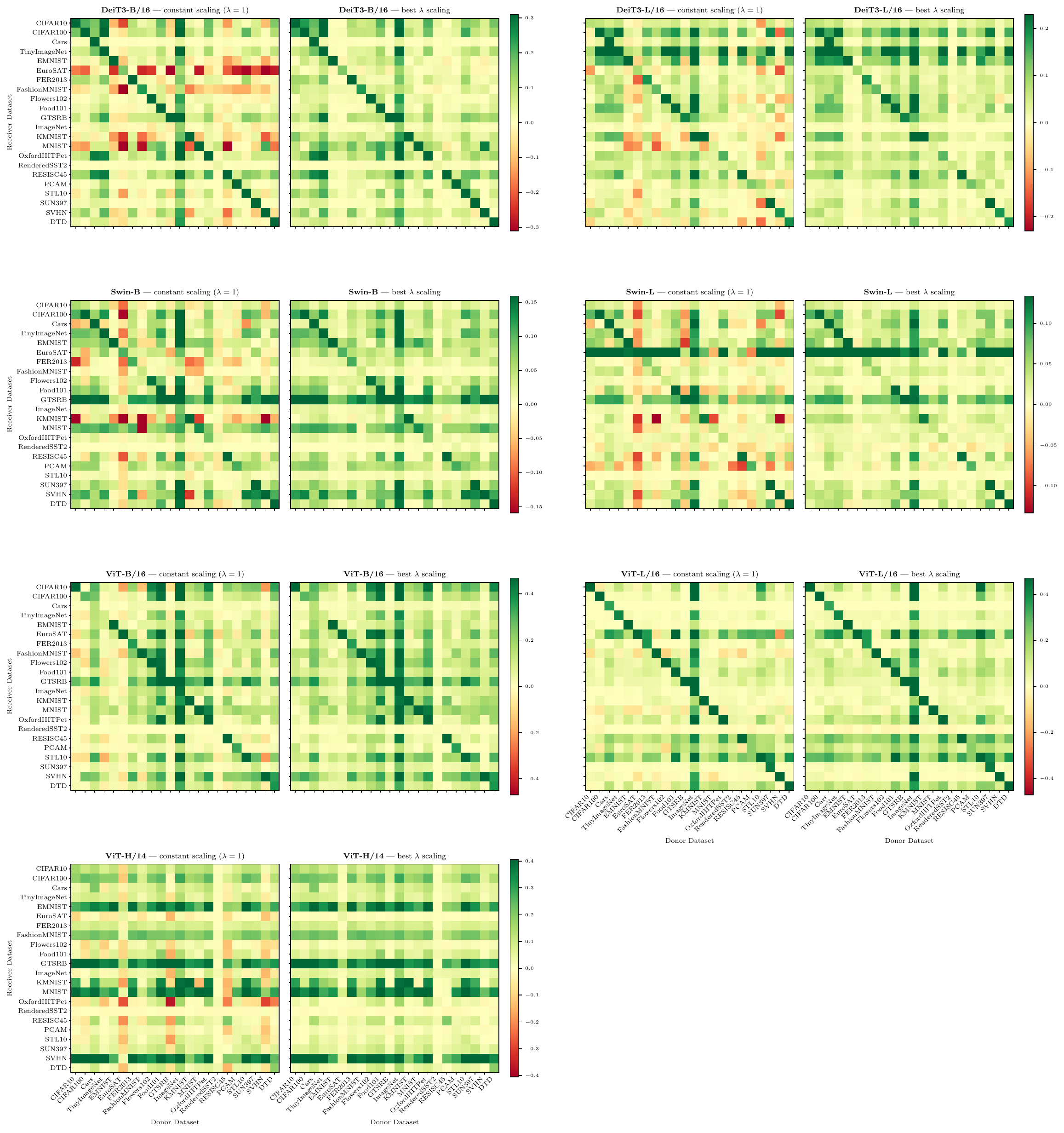}
    \caption{\textbf{QV transferability for vision (learned-head)}. Top-1 accuracy change ($\Delta$) from patching receiver $r$ with donor $d$ quantization vector, relative to baseline 3-bit PTQ. For each architecture, heatmaps on the left show transfer using a constant scaling factor ($\lambda=1$), revealing both positive transfer and destructive interference. Heatmaps on the right demonstrate that modulating the magnitude $\lambda$ eliminates destructive interference while maximizing gains.}
    \label{fig:timm_supervised_alpha_unit_vs_best_all_models}
\end{figure}
\begin{figure}[]
    \centering
    \includegraphics[width=0.99\linewidth]{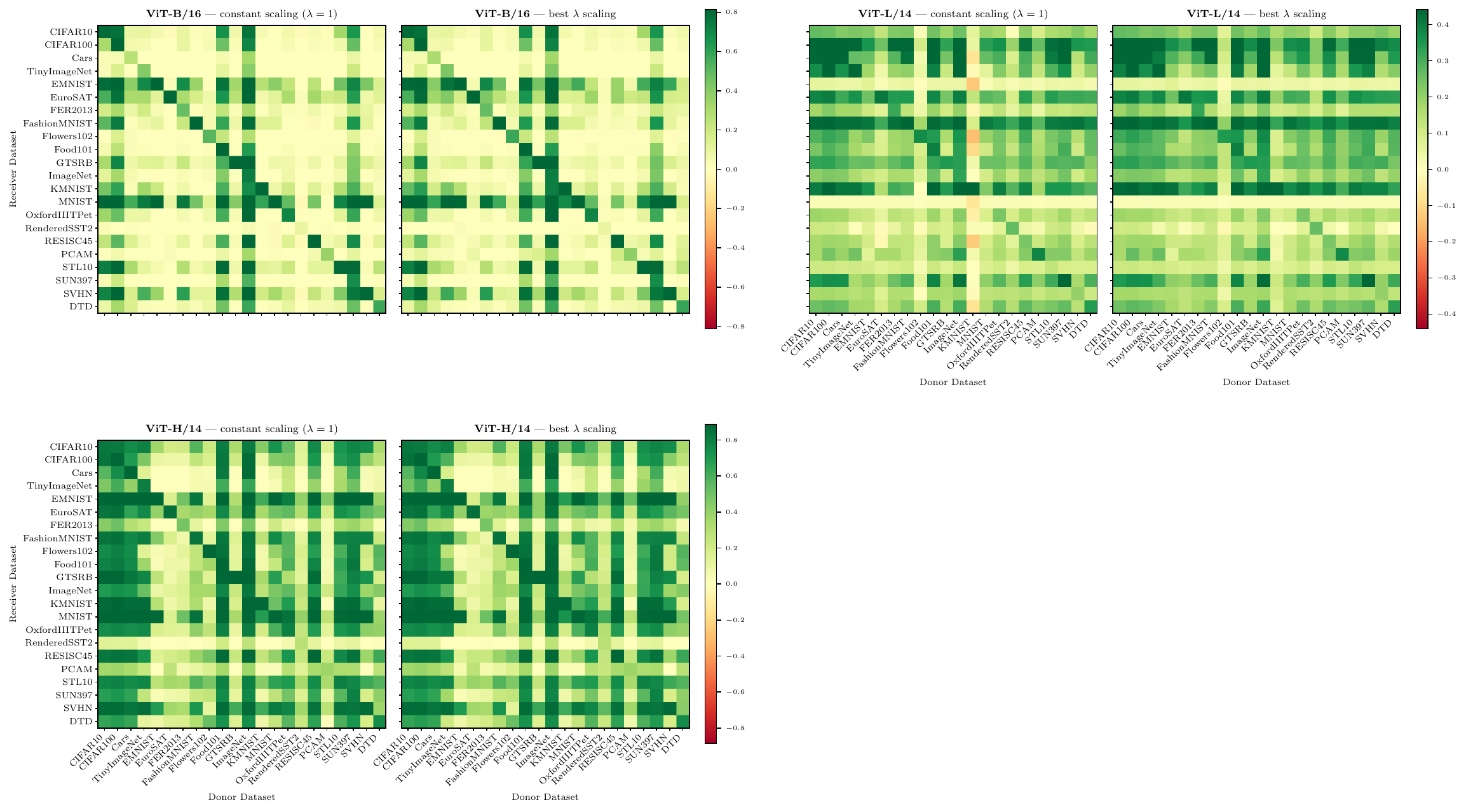}
    \caption{\textbf{QV transferability for vision (CLIP-style)}. Top-1 accuracy change ($\Delta$) from patching receiver $r$ with donor $d$ quantization vector, relative to baseline 3-bit PTQ. For each architecture, heatmaps on the left show transfer using a constant scaling factor ($\lambda=1$), revealing both positive transfer and destructive interference. Heatmaps on the right demonstrate that modulating the magnitude $\lambda$ eliminates destructive interference while maximizing gains.}
    \label{fig:open_clip_alpha_unit_vs_best_all_models}
\end{figure}
\begin{figure}[]
    \centering
    \includegraphics[width=0.99\linewidth]{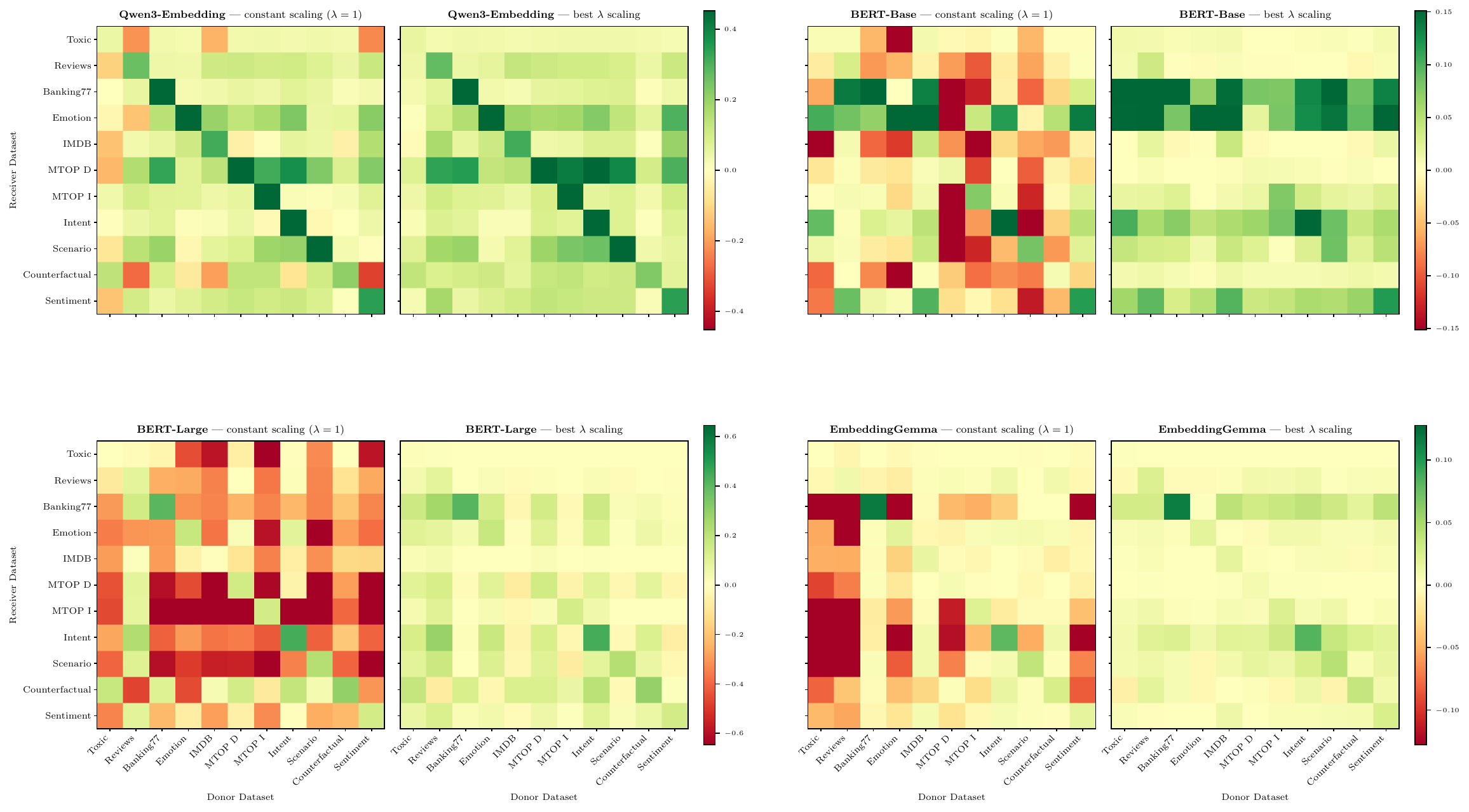}
    \caption{\textbf{QV transferability for text}. Top-1 accuracy change ($\Delta$) from patching receiver $r$ with donor $d$ quantization vector, relative to baseline 3-bit PTQ. For each architecture, heatmaps on the left show transfer using a constant scaling factor ($\lambda=1$), revealing both positive transfer and destructive interference. Heatmaps on the right demonstrate that modulating the magnitude $\lambda$ eliminates destructive interference while maximizing gains.}
    \label{fig:automodelforsequenceclassification_alpha_unit_vs_best_all_models}
\end{figure}

In this appendix, we provide the full set of cross-task QV transfer results
that complement the main analysis in Section~\ref{sec:results}. The main text
focuses on representative heatmaps and aggregate trends, while here we report
the corresponding results across all remaining architectures, head regimes, and
modalities.

\paragraph{Unit-scale transfer across all models.}

For vision models, unit-scale QV transfer is broadly beneficial. In the
CLIP-style setting, all reported backbones have positive mean transfer, with
average gains ranging from \(+18.7\%\) to \(+43.7\%\) and
positive-transfer rates above \(92\%\). The strongest gains appear for the
largest CLIP-style model, \texttt{ViT-H/14}, whose mean unit-scale transfer reaches
\(+43.7\%\). Learned-head vision models show the same
qualitative trend, although with smaller and less uniform gains. Their mean
unit-scale transfer remains positive for all architectures, ranging from
\(+1.3\%\) to \(+9.7\%\), with positive-transfer rates between
\(60.8\%\) and \(76.6\%\). These results indicate that, in vision, the raw QV
direction often already moves the receiver toward a more quantization-robust
region of weight space.

The text models exhibit a more heterogeneous pattern at unit scale. \texttt{Qwen3-Embedding} shows strong positive transfer, with a mean gain of \(+4.8\%\)
and \(78.2\%\) on positive-transfer pairs. By contrast,
\texttt{BERT-base}, \texttt{BERT-large}, and \texttt{EmbeddingGemma} have negative mean transfer at
\(\lambda=1\). The strongest failure case is \texttt{BERT-large}, whose unit-scale mean
transfer is \(-29.1\%\). This does not imply that text QVs are
not transferable; rather, it shows that applying the donor displacement with a
fixed unit magnitude can be unreliable in some text encoders. This observation
is consistent with Proposition~\ref{prop:alignment_controls_transfer}, which
predicts that a useful donor direction may still hurt if its scale is not
matched to the receiver geometry.

\paragraph{Effect of scale calibration.}
We next evaluate whether the negative transfer observed at \(\lambda=1\) can be
mitigated by tuning only the scalar coefficient in
Eq.~\ref{eq:qv_patching}. For each donor--receiver pair, we select the best
coefficient from the same fixed grid used in the main text and then evaluate the
patched receiver after 3-bit PTQ. We denote this selected coefficient by
\(\lambda^{\mathrm{best}}\).

The right-hand heatmaps in
Figures~\ref{fig:qv_transfer_heatmap_unit_vs_alpha_scale_timm_supervised_vit_b_16},
\ref{fig:qv_transfer_heatmap_unit_vs_alpha_scale_automodelforsequenceclassification_qwen3_embedding},
\ref{fig:timm_supervised_alpha_unit_vs_best_all_models},
\ref{fig:open_clip_alpha_unit_vs_best_all_models}, and
\ref{fig:automodelforsequenceclassification_alpha_unit_vs_best_all_models}
show the resulting transfer matrices after scale calibration on validation data. Compared with the
unit-scale setting, the calibrated setting substantially reduces destructive
interference and increases the reliability of transfer. This is especially
visible in the text experiments, where several models that were negative on
average at \(\lambda=1\) become positive after scale selection.

In learned-head vision models, the mean gain becomes
positive for every architecture and increases relative to unit-scale transfer.
For example, DeiT3-B/16 improves from a mean unit-scale gain of \(+1.7\%\) to \(+6.1\%\), \texttt{Swin-B} from \(+2.3\%\) to \(+4.3\%\), \texttt{\texttt{ViT-B/16}} from
\(+9.7\) to \(+13.2\), and \texttt{ViT-H/14} from \(+8.5\) to \(+11.2\). Positive-transfer
rates also increase substantially, reaching between \(87.7\%\) and \(95.5\%\)
across learned-head vision architectures.

The CLIP-style models remain the strongest regime. Scale calibration preserves
the large gains already observed at unit scale and further improves reliability.
The mean calibrated gains range from \(+18.1\%\) to \(+44.8\%\),
with positive-transfer rates between \(87.9\%\) and \(99.8\%\). In particular,
\texttt{ViT-H/14} reaches a mean gain of \(+44.8\%\), while \texttt{ViT-L/14}
achieves positive transfer in \(99.8\%\) of donor--receiver pairs. These results
suggest that CLIP-style visual backbones contain especially transferable
quantization-robust directions.

The effect of scale calibration is most pronounced in the text setting.
\texttt{BERT-base} changes from a negative mean transfer of \(-3.6\%\) at
unit scale to a positive mean transfer of \(+3.9\%\), while its positive-transfer
rate increases from \(35.5\%\) to \(85.5\%\). \texttt{BERT-large} changes from
\(-29.1\%\) to \(+4.0\), and its positive-transfer rate increases from \(15.5\%\)
to \(70.9\%\). \texttt{\texttt{EmbeddingGemma}} also becomes positive on average, although with a
smaller mean gain of \(+0.7\%\). \texttt{Qwen3-Embedding} remains the strongest text model,
improving from \(+4.8\%\) to \(+10.0\%\) mean transfer and reaching \(99.1\%\)
positive-transfer pairs. 
These results show that many negative unit-scale text transfers are not caused by an absence of transferable QV structure. To separate a wrong step size from a genuinely anti-aligned direction, we extend the \(\lambda\) sweep to negative values for \texttt{BERT-base} and \texttt{BERT-large}. The failures turn out to be mixed: for most failing pairs the best scale is positive but smaller than one, i.e.\ a magnitude mismatch, whereas for a minority, larger in \texttt{BERT-large}, the best scale is negative. The latter case is still consistent with Proposition~\ref{prop:alignment_controls_transfer}, whose recoverable fraction \(\cos^2_{H_{\mathcal R}}(\rho_{\mathcal D},\rho_{\mathcal R})\) does not depend on the sign of the donor QV: a negative optimum only means that the donor direction must be applied with the opposite sign. Transfer is truly impossible only when the two QVs are nearly orthogonal in the receiver's geometry.

\paragraph{Amortized Deployment Cost.}
\label{par:amortized_cost}

QV reuse incurs a one-time QAT cost on the donor, after which the resulting vector can be applied to multiple compatible receivers. 
In contrast, the standard approach requires a separate QAT run for every receiver. We quantify this difference by the number of training samples processed during QAT. 
For a dataset \(D\), let
\[
C_{\mathrm{QAT}}(D)=E_D\lvert D_{\mathrm{train}}\rvert
\]
denote the QAT cost, where \(E_D\) is the number of QAT epochs and \(\lvert D_{\mathrm{train}}\rvert\) is the number of training examples. 
Because donor and receiver runs use the same backbone and training procedure, this measure isolates the computational cost that differs between QV reuse and independent receiver-side QAT.

For \(N\) receiver deployments, the total QAT-related cost of QV reuse is
\[
C_{\mathrm{reuse}}(N)
=
C_{\mathrm{QAT}}(D_{\mathrm{donor}})
+
N c_{\mathrm{patch}},
\]
where \(c_{\mathrm{patch}}\) is the cost of applying the stored vector to one receiver. By comparison, performing QAT independently for each receiver costs
\[
C_{\mathrm{separate}}(N)
=
\sum_{i=1}^{N} C_{\mathrm{QAT}}(D_i).
\]
The break-even deployment count is therefore the smallest \(N\) for which QV reuse is no more expensive than independent receiver-side QAT:
\[
N^\star
=
\min\!\left\{
N\in\mathbb{N}:
C_{\mathrm{reuse}}(N)
\leq
C_{\mathrm{separate}}(N)
\right\}.
\]
If \(\bar C\) denotes the mean QAT cost across receivers, this can be approximated as
\[
N^\star
\approx
\left\lceil
\frac{C_{\mathrm{QAT}}(D_{\mathrm{donor}})}
{\bar C-c_{\mathrm{patch}}}
\right\rceil .
\]

In practice, applying a QV only requires a state-dict addition, whose cost is negligible relative to QAT, so \(c_{\mathrm{patch}}\approx0\). 
On the vision datasets, this yields a break-even point of \(N\geq6\) when ImageNet is used as the donor and \(N\geq3\) for the other evaluated donors. 
Thus, once the corresponding number of compatible receivers is reached, the initial donor QAT cost has been amortized; each additional receiver avoids an otherwise separate QAT run.

\section{QV Patching with Stronger PTQ Methods}
\label{app:strong_ptq}
We study how QV patching interacts with stronger post-training quantizers, namely AWQ \cite{lin2024awq} and GPTQ \cite{frantar2022gptq}, on the learned-head \texttt{ViT-B/16} backbone with 3-bit symmetric per-channel weight-only quantization. Throughout, RTN denotes the fake-quantization operator of Eq.~\ref{eq:quantize_dequantize_roundtrip}, i.e.\ the vanilla PTQ used in the rest of the paper.

\paragraph{QV patching is not a competitor to strong PTQ.} AWQ and GPTQ are far stronger than RTN, and QV patching followed by RTN does not reach their accuracy on any donor--receiver pair. The contribution of QV patching is not state-of-the-art low-bit accuracy, but the demonstration that quantization robustness can be transferred to new tasks without receiver-side QAT.

\paragraph{Strong PTQ behaves normally on patched weights.} A QV-patched checkpoint remains a valid input for strong quantizers: replacing RTN with AWQ or GPTQ on a patched receiver improves accuracy on virtually every cross-task pair, by a margin comparable to what these methods achieve on unpatched weights. The only exception is the same-task case, where the patch at unit scale coincides with the receiver's own QAT checkpoint: its weights already lie on the RTN grid, and methods that deliberately move weights off that grid remove the property QAT installed.

\paragraph{Does the QV help on top of AWQ or GPTQ?} Comparing against the same strong quantizer applied to the unpatched receiver, unit-scale patching hurts under both AWQ and GPTQ. This is consistent with Proposition~\ref{prop:alignment_controls_transfer}: changing the final quantizer changes the receiver's post-quantization objective, hence \(H_{\mathcal R}\) and the optimal scale \(\lambda^\star\). When \(\lambda\) is selected on receiver validation data, the QV yields a small but consistent gain on top of AWQ, and no meaningful gain on top of GPTQ, where improvements on most receivers are offset by a few large losses. The benefit of the QV is therefore bounded by the residual error the deployment quantizer leaves: in the strictly data-free setting, QV patching does not help on top of AWQ or GPTQ, and the zero-shot claim of this paper refers to RTN deployment.

\paragraph{AWQ and GPTQ displacements do not transfer.} As an ablation, we replace the QAT-derived QV with the weight displacement that AWQ or GPTQ induces on the donor, and add it to a different task's checkpoint. Both transferred displacements degrade the receiver, both relative to its own AWQ/GPTQ result and relative to the unpatched receiver. We hypothesize that this is because AWQ and GPTQ are activation-space methods: although they output modified weights, their corrections are dictated by the donor's activation statistics on calibration data, through activation-aware channel scaling in AWQ and the input Hessian used for error compensation in GPTQ, and are therefore specific to the donor's input distribution. QAT, in contrast, shapes the weights directly in weight space, which is consistent with its displacement being the one that transfers across tasks.

\section{QV Transfer at Other Bit Widths}
\label{app:bit_widths}
QV transfer is largely accuracy-preserving at both extremes of the bit-width range and strongly beneficial in the intermediate regime, where quantization causes substantial degradation that receiver-side QAT can still undo. This is consistent with Proposition~\ref{prop:alignment_controls_transfer}, which bounds the gain recoverable by a donor patch by the receiver's own QAT gain. We obtain this picture by repeating the cross-task transfer setup of Section~\ref{sec:results} at 2 and 4 bits, with the same quantization configuration (Section~\ref{sec:exp_setup}) and the strictly data-free unit-scale protocol (\(\lambda=1\)).
At 4 bits, where PTQ leaves little room for improvement (Section~\ref{sec:results}), QV transfer brings no aggregate gain, but nearly all patched receivers stay close to their vanilla PTQ accuracy and none collapses.
At 3 bits, the regime studied in the main text, QV transfer yields the gains reported in Section~\ref{sec:results}.
At 2 bits, vanilla PTQ collapses severely, destroying most of the full-precision accuracy, and even receiver-side QAT recovers only a limited share of it. Since a donor QV can at best transfer the robustness that QAT itself achieves, a weak QAT endpoint yields a weak QV: patching leaves the receiver near its poor PTQ baseline, and the transfer inherits the failure of QAT in this regime.

\section{Empirical Check of Proposition~\ref{prop:alignment_controls_transfer}}
\label{app:alignment_check}
Proposition~\ref{prop:alignment_controls_transfer} predicts that transfer is governed by the alignment between donor and receiver QVs. We measure this alignment with the Euclidean cosine between \(\rho_{\mathcal D}\) and \(\rho_{\mathcal R}\) of Equation~\ref{eq:quantization_vector}, computed on the linear-layer weights that PTQ acts on, and correlate it with the QV patching gain over all cross-task pairs of the setup of Section~\ref{sec:results}, at each \(\lambda\) of the grid.
We find a clear positive correlation between donor--receiver QV alignment and QV patching gain at small \(\lambda\), and then a decreasing correlation as \(\lambda\) grows, crossing zero near \(\lambda=1\).
We test the validity of this correlation by computing the cosine row-wise instead of globally, by leaving out each task in turn, and by a Quadratic Assignment Procedure test, which accounts for pairs sharing a task: the results are preserved in all cases and are statistically significant.
This shows that well-aligned pairs benefit most from small \(\lambda\), while poorly aligned pairs need larger ones; at \(\lambda=1\) the two effects cancel out, which explains the near-zero correlation there.
Finally, we note that although raw cosine similarities are small, they lie far above what unrelated directions would produce in a space of this dimension, indicating that donor and receiver QVs share a non-trivial common component.

\section{Does QV Patching Reduce the Quantization Error?}
\label{app:quant_error}
A natural explanation for QV transfer is that the patch moves the weights closer to the quantization grid, making them easier to quantize. We test this directly by measuring the relative weight-space quantization error \(\|W-\mathrm{FQ}(W)\|/\|W\|\) of patched receivers over all cross-task pairs of the setup of Section~\ref{sec:results}, both at unit scale and at the validation-selected \(\lambda\). The error is essentially unchanged by patching, and it is also essentially unchanged in the receiver's own QAT checkpoint, which is nonetheless far more robust to PTQ; in an ablation comparison, the QV affects this error no differently than a random direction of the same norm.
QV patching therefore does not appear to work by shrinking the quantization error. We hypothesize that it instead changes where that error falls relative to the network function, steering the weights toward configurations whose quantization residual lies along directions to which the loss is less sensitive. This would be consistent with the local picture of Proposition~\ref{prop:alignment_controls_transfer}, in which the post-quantization damage depends on the curvature of the loss along the error, not only on its magnitude; verifying it, e.g.\ through curvature-aware analyses, is left for future work.

\section{Further Variants of QV Transfer}
\label{app:variants}

\paragraph{QAT optimizer.} To check that QV transfer is not an artifact of the straight-through estimator~\citep{bengio2013estimating} used for all other QAT runs in this paper, we rebuild the donor QVs with PV-Tuning~\citep{malinovskii2024pv}, following its original specification and taking the full-precision weights it optimizes, before their final rounding onto the quantization grid, and repeat the cross-task transfer setup of Section~\ref{sec:results} at unit scale (\(\lambda=1\)). PV-Tuning reproduces the same transfer phenomenon almost exactly, with nearly identical gains and positive-transfer rates; the two QVs are highly aligned and have comparable norms. PV-Tuning thus does not improve transfer, but it shows that the QV is a property of quantization-aware training rather than of the particular way QAT is implemented.

\paragraph{Staged schedules.} The main experiments train the FT and QAT endpoints side by side from the pretrained backbone. We also tested a staged variant: quantization-aware training starts from an intermediate full-precision checkpoint, at 30, 60 or 90\% of the full-precision run, and runs for the remaining steps with its own learning-rate schedule, so the total number of optimization steps stays fixed. The QV is then the displacement between the final quantization-aware and final full-precision checkpoints. Across the 22 vision tasks with \texttt{ViT-B/16}~\citep{dosovitskiy2020image} at 3 bits and unit scale (\(\lambda=1\)), QVs from staged donors transfer consistently, recovering a large share of the receiver's own QA gain. Within this set of runs, staged donors transfer markedly more strongly than donors trained quantization-aware from the pretrained backbone with the same budget, which suggests that QVs extracted closer to the full-precision solution are more reusable.

\paragraph{Quantization-aware distillation.} Under the same staged schedules, we also trained donors with quantization-aware distillation (QAD)~\citep{hinton2015distilling, xin2026quantizationawaredistillation}, using the full-precision checkpoint as teacher. QVs from QAD donors transfer about as well as those from QAT donors.

\paragraph{Beyond classification.} To check that QV transfer is not specific to discriminative models, we applied it to the decoder of the Stable Diffusion autoencoder (KL-f8)~\citep{rombach2022high} finetuned on 12 image domains: CelebA-HQ~\citep{karras2017progressive}, ChartQA~\citep{masry-etal-2022-chartqa}, chest X-rays~\citep{wang2017chestxray}, COCO~\citep{lin2014microsoft}, floor plans~\citep{fan2021floorplancad}, Food-101~\citep{bossard_food-101_2014}, LSUN Bedroom and Church~\citep{yu2015lsun}, pixel art~\citep{pixelart_308k}, PubLayNet~\citep{zhong2019publaynet}, RESISC45~\citep{cheng_remote_2017} and RICO~\citep{deka2017rico, wang2021screen2words}. Under 3-bit weight PTQ, a QV from a different domain improves the receiver's reconstruction in most donor-receiver pairs and closes roughly half of the gap between vanilla PTQ and the receiver's own QAT. The gains carry over to 4-bit floating-point formats, and QAT and QAD donors are again indistinguishable.

\section{Societal Impacts}
\label{sec:societal_impacts}

\paragraph{Positive societal impacts.} By enabling zero-shot transfer of quantization robustness, this work reduces the compute, data, and engineering costs of deploying models at extremely low bit-widths. This can democratize efficient deployment for resource-constrained practitioners, lower the energy footprint of AI inference, and avoid the need to access sensitive task-specific data on the receiver side, a meaningful privacy advantage in domains like healthcare.

\paragraph{Negative societal impacts.} Cheaper deployment of compact models could accelerate adoption in sensitive domains (e.g., surveillance, automated decision-making) without proportionate investment in safety or fairness evaluation. Because patching is zero-shot and data-free, practitioners may deploy models with silently degraded accuracy on critical tasks if transfer quality is not carefully validated. Extensions beyond vision classifiers could amplify these risks.

\section{Safeguards}
\label{sec:safeguards}

The artifacts released with this work consist of quantization vectors, i.e. lightweight weight-space displacements, and the corresponding fine-tuned checkpoints for standard image classification tasks. 
These are derived from publicly available pretrained, well-established, and openly licensed classification models and datasets. 
As such, they do not introduce new risks of misuse beyond those already present in the underlying models and datasets. 
The quantization vectors themselves encode structural robustness to low-bit quantization rather than new task capabilities, and cannot be used in isolation to generate content or perform inference. 
We do \textit{not} release large-scale generative models, scraped datasets, or personally identifiable data.

\FloatBarrier
\newpage
\input{checklist.tex}

\end{document}

%% file: checklist.tex
\section*{NeurIPS Paper Checklist}

\begin{enumerate}

\item {\bf Claims}
    \item[] Question: Do the main claims made in the abstract and introduction accurately reflect the paper's contributions and scope?
    \item[] Answer: \answerYes{}.
    \item[] Justification: The four main claims, namely (1) quantization vectors as transferable displacements, (2) well-definedness under symmetries, (3) cosine-squared recovery theory, and (4) empirical gains up to \(\sim\)60 points across 4 ViT scales and 22 tasks, are each substantiated: Theorem 1 (Section~\ref{sec:qv} and Appendix~\ref{app:wd}) supports (2), Proposition~\ref{prop:alignment_controls_transfer} (Appendix \ref{app:proof_cosine_squared_recovery}) supports (3), and the heatmaps and tables in Section~\ref{sec:exp_setup} and Appendix~\ref{app:extended_results} support (1) and (4).
    \item[] Guidelines:
    \begin{itemize}
        \item The answer \answerNA{} means that the abstract and introduction do not include the claims made in the paper.
        \item The abstract and/or introduction should clearly state the claims made, including the contributions made in the paper and important assumptions and limitations. A \answerNo{} or \answerNA{} answer to this question will not be perceived well by the reviewers. 
        \item The claims made should match theoretical and experimental results, and reflect how much the results can be expected to generalize to other settings. 
        \item It is fine to include aspirational goals as motivation as long as it is clear that these goals are not attained by the paper. 
    \end{itemize}

\item {\bf Limitations}
    \item[] Question: Does the paper discuss the limitations of the work performed by the authors?
    \item[] Answer: \answerYes{}.
    \item[] Justification: The paper includes a dedicated "Limitations" section (Appendix \ref{sec:limitations}) that addresses several concrete limitations.
    \item[] Guidelines:
    \begin{itemize}
        \item The answer \answerNA{} means that the paper has no limitation while the answer \answerNo{} means that the paper has limitations, but those are not discussed in the paper. 
        \item The authors are encouraged to create a separate ``Limitations'' section in their paper.
        \item The paper should point out any strong assumptions and how robust the results are to violations of these assumptions (e.g., independence assumptions, noiseless settings, model well-specification, asymptotic approximations only holding locally). The authors should reflect on how these assumptions might be violated in practice and what the implications would be.
        \item The authors should reflect on the scope of the claims made, e.g., if the approach was only tested on a few datasets or with a few runs. In general, empirical results often depend on implicit assumptions, which should be articulated.
        \item The authors should reflect on the factors that influence the performance of the approach. For example, a facial recognition algorithm may perform poorly when image resolution is low or images are taken in low lighting. Or a speech-to-text system might not be used reliably to provide closed captions for online lectures because it fails to handle technical jargon.
        \item The authors should discuss the computational efficiency of the proposed algorithms and how they scale with dataset size.
        \item If applicable, the authors should discuss possible limitations of their approach to address problems of privacy and fairness.
        \item While the authors might fear that complete honesty about limitations might be used by reviewers as grounds for rejection, a worse outcome might be that reviewers discover limitations that aren't acknowledged in the paper. The authors should use their best judgment and recognize that individual actions in favor of transparency play an important role in developing norms that preserve the integrity of the community. Reviewers will be specifically instructed to not penalize honesty concerning limitations.
    \end{itemize}

\item {\bf Theory assumptions and proofs}
    \item[] Question: For each theoretical result, does the paper provide the full set of assumptions and a complete (and correct) proof?
    \item[] Answer: \answerYes{}.
    \item[] Justification: All theoretical results (Theorems~\ref{thm:qv_well_defined}–\ref{thm:adamw_signed_permutation}, Propositions~\ref{prop:alignment_controls_transfer}–\ref{prop:cosine_squared_second_order}, Lemma~\ref{lem:fq_commutes}) explicitly state their assumptions and are accompanied by complete proofs in the appendices (Appendices~\ref{app:proof_cosine_squared_recovery},~\ref{sub:first_deviation_in_local_displacement},~\ref{app:wd}), with proof sketches provided in the main text for the key results.
    \item[] Guidelines:
    \begin{itemize}
        \item The answer \answerNA{} means that the paper does not include theoretical results. 
        \item All the theorems, formulas, and proofs in the paper should be numbered and cross-referenced.
        \item All assumptions should be clearly stated or referenced in the statement of any theorems.
        \item The proofs can either appear in the main paper or the supplemental material, but if they appear in the supplemental material, the authors are encouraged to provide a short proof sketch to provide intuition. 
        \item Inversely, any informal proof provided in the core of the paper should be complemented by formal proofs provided in appendix or supplemental material.
        \item Theorems and Lemmas that the proof relies upon should be properly referenced. 
    \end{itemize}

    \item {\bf Experimental result reproducibility}
    \item[] Question: Does the paper fully disclose all the information needed to reproduce the main experimental results of the paper to the extent that it affects the main claims and/or conclusions of the paper (regardless of whether the code and data are provided or not)?
    \item[] Answer: \answerYes{}.
    \item[] Justification: Architectures, quantization scheme, optimizer, datasets, and other fundamental reproducibility details are presented in Section~\ref{sec:exp_setup}.
    \item[] Guidelines:
    \begin{itemize}
        \item The answer \answerNA{} means that the paper does not include experiments.
        \item If the paper includes experiments, a \answerNo{} answer to this question will not be perceived well by the reviewers: Making the paper reproducible is important, regardless of whether the code and data are provided or not.
        \item If the contribution is a dataset and\slash or model, the authors should describe the steps taken to make their results reproducible or verifiable. 
        \item Depending on the contribution, reproducibility can be accomplished in various ways. For example, if the contribution is a novel architecture, describing the architecture fully might suffice, or if the contribution is a specific model and empirical evaluation, it may be necessary to either make it possible for others to replicate the model with the same dataset, or provide access to the model. In general. releasing code and data is often one good way to accomplish this, but reproducibility can also be provided via detailed instructions for how to replicate the results, access to a hosted model (e.g., in the case of a large language model), releasing of a model checkpoint, or other means that are appropriate to the research performed.
        \item While NeurIPS does not require releasing code, the conference does require all submissions to provide some reasonable avenue for reproducibility, which may depend on the nature of the contribution. For example
        \begin{enumerate}
            \item If the contribution is primarily a new algorithm, the paper should make it clear how to reproduce that algorithm.
            \item If the contribution is primarily a new model architecture, the paper should describe the architecture clearly and fully.
            \item If the contribution is a new model (e.g., a large language model), then there should either be a way to access this model for reproducing the results or a way to reproduce the model (e.g., with an open-source dataset or instructions for how to construct the dataset).
            \item We recognize that reproducibility may be tricky in some cases, in which case authors are welcome to describe the particular way they provide for reproducibility. In the case of closed-source models, it may be that access to the model is limited in some way (e.g., to registered users), but it should be possible for other researchers to have some path to reproducing or verifying the results.
        \end{enumerate}
    \end{itemize}

\item {\bf Open access to data and code}
    \item[] Question: Does the paper provide open access to the data and code, with sufficient instructions to faithfully reproduce the main experimental results, as described in supplemental material?
    \item[] Answer: \answerYes{}.
    \item[] Justification: All datasets used are publicly available, we provide the code in supplementary materials, and we plan to release code and checkpoints on GitHub upon acceptance.
    \item[] Guidelines:
    \begin{itemize}
        \item The answer \answerNA{} means that paper does not include experiments requiring code.
        \item Please see the NeurIPS code and data submission guidelines (\url{https://neurips.cc/public/guides/CodeSubmissionPolicy}) for more details.
        \item While we encourage the release of code and data, we understand that this might not be possible, so \answerNo{} is an acceptable answer. Papers cannot be rejected simply for not including code, unless this is central to the contribution (e.g., for a new open-source benchmark).
        \item The instructions should contain the exact command and environment needed to run to reproduce the results. See the NeurIPS code and data submission guidelines (\url{https://neurips.cc/public/guides/CodeSubmissionPolicy}) for more details.
        \item The authors should provide instructions on data access and preparation, including how to access the raw data, preprocessed data, intermediate data, and generated data, etc.
        \item The authors should provide scripts to reproduce all experimental results for the new proposed method and baselines. If only a subset of experiments are reproducible, they should state which ones are omitted from the script and why.
        \item At submission time, to preserve anonymity, the authors should release anonymized versions (if applicable).
        \item Providing as much information as possible in supplemental material (appended to the paper) is recommended, but including URLs to data and code is permitted.
    \end{itemize}

\item {\bf Experimental setting/details}
    \item[] Question: Does the paper specify all the training and test details (e.g., data splits, hyperparameters, how they were chosen, type of optimizer) necessary to understand the results?
    \item[] Answer: \answerYes{}.
    \item[] Justification: The full details for the reproduction of the experiments are provided in the code. 
    \item[] Guidelines:
    \begin{itemize}
        \item The answer \answerNA{} means that the paper does not include experiments.
        \item The experimental setting should be presented in the core of the paper to a level of detail that is necessary to appreciate the results and make sense of them.
        \item The full details can be provided either with the code, in appendix, or as supplemental material.
    \end{itemize}

\item {\bf Experiment statistical significance}
    \item[] Question: Does the paper report error bars suitably and correctly defined or other appropriate information about the statistical significance of the experiments?
    \item[] Answer: \answerNo{}.
    \item[] Justification: Given a fixed compute budget, we opted to invest it in breadth across multiple datasets and model scales rather than in multiple seeds per dataset. Error bars over different seeds would require substantial additional computation on each dataset; instead, the breadth of the evaluation provides robustness evidence.
    \item[] Guidelines:
    \begin{itemize}
        \item The answer \answerNA{} means that the paper does not include experiments.
        \item The authors should answer \answerYes{} if the results are accompanied by error bars, confidence intervals, or statistical significance tests, at least for the experiments that support the main claims of the paper.
        \item The factors of variability that the error bars are capturing should be clearly stated (for example, train/test split, initialization, random drawing of some parameter, or overall run with given experimental conditions).
        \item The method for calculating the error bars should be explained (closed form formula, call to a library function, bootstrap, etc.)
        \item The assumptions made should be given (e.g., Normally distributed errors).
        \item It should be clear whether the error bar is the standard deviation or the standard error of the mean.
        \item It is OK to report 1-sigma error bars, but one should state it. The authors should preferably report a 2-sigma error bar than state that they have a 96\% CI, if the hypothesis of Normality of errors is not verified.
        \item For asymmetric distributions, the authors should be careful not to show in tables or figures symmetric error bars that would yield results that are out of range (e.g., negative error rates).
        \item If error bars are reported in tables or plots, the authors should explain in the text how they were calculated and reference the corresponding figures or tables in the text.
    \end{itemize}

\item {\bf Experiments compute resources}
    \item[] Question: For each experiment, does the paper provide sufficient information on the computer resources (type of compute workers, memory, time of execution) needed to reproduce the experiments?
    \item[] Answer: \answerYes{}.
    \item[] Justification: Appendix~\ref{app:compute_setup} provides details about compute resources used during the experimental phase.
    \item[] Guidelines:
    \begin{itemize}
        \item The answer \answerNA{} means that the paper does not include experiments.
        \item The paper should indicate the type of compute workers CPU or GPU, internal cluster, or cloud provider, including relevant memory and storage.
        \item The paper should provide the amount of compute required for each of the individual experimental runs as well as estimate the total compute. 
        \item The paper should disclose whether the full research project required more compute than the experiments reported in the paper (e.g., preliminary or failed experiments that didn't make it into the paper). 
    \end{itemize}
    
\item {\bf Code of ethics}
    \item[] Question: Does the research conducted in the paper conform, in every respect, with the NeurIPS Code of Ethics \url{https://neurips.cc/public/EthicsGuidelines}?
    \item[] Answer: \answerYes{}.
    \item[] Justification: The authors have reviewed the NeurIPS Code of Ethics and believe the manuscript is aligned with it.
    \item[] Guidelines:
    \begin{itemize}
        \item The answer \answerNA{} means that the authors have not reviewed the NeurIPS Code of Ethics.
        \item If the authors answer \answerNo, they should explain the special circumstances that require a deviation from the Code of Ethics.
        \item The authors should make sure to preserve anonymity (e.g., if there is a special consideration due to laws or regulations in their jurisdiction).
    \end{itemize}

\item {\bf Broader impacts}
    \item[] Question: Does the paper discuss both potential positive societal impacts and negative societal impacts of the work performed?
    \item[] Answer: \answerYes{}.
    \item[] Justification: Section~\ref{sec:societal_impacts} shows potential positive and negative societal impacts stemming from this paper.
    \item[] Guidelines:
    \begin{itemize}
        \item The answer \answerNA{} means that there is no societal impact of the work performed.
        \item If the authors answer \answerNA{} or \answerNo, they should explain why their work has no societal impact or why the paper does not address societal impact.
        \item Examples of negative societal impacts include potential malicious or unintended uses (e.g., disinformation, generating fake profiles, surveillance), fairness considerations (e.g., deployment of technologies that could make decisions that unfairly impact specific groups), privacy considerations, and security considerations.
        \item The conference expects that many papers will be foundational research and not tied to particular applications, let alone deployments. However, if there is a direct path to any negative applications, the authors should point it out. For example, it is legitimate to point out that an improvement in the quality of generative models could be used to generate Deepfakes for disinformation. On the other hand, it is not needed to point out that a generic algorithm for optimizing neural networks could enable people to train models that generate Deepfakes faster.
        \item The authors should consider possible harms that could arise when the technology is being used as intended and functioning correctly, harms that could arise when the technology is being used as intended but gives incorrect results, and harms following from (intentional or unintentional) misuse of the technology.
        \item If there are negative societal impacts, the authors could also discuss possible mitigation strategies (e.g., gated release of models, providing defenses in addition to attacks, mechanisms for monitoring misuse, mechanisms to monitor how a system learns from feedback over time, improving the efficiency and accessibility of ML).
    \end{itemize}
    
\item {\bf Safeguards}
    \item[] Question: Does the paper describe safeguards that have been put in place for responsible release of data or models that have a high risk for misuse (e.g., pre-trained language models, image generators, or scraped datasets)?
    \item[] Answer: \answerYes{}
    \item[] Justification: Section~\ref{sec:safeguards} describes safeguards that have been put in place to address modern AI concerns.
    \item[] Guidelines:
    \begin{itemize}
        \item The answer \answerNA{} means that the paper poses no such risks.
        \item Released models that have a high risk for misuse or dual-use should be released with necessary safeguards to allow for controlled use of the model, for example by requiring that users adhere to usage guidelines or restrictions to access the model or implementing safety filters. 
        \item Datasets that have been scraped from the Internet could pose safety risks. The authors should describe how they avoided releasing unsafe images.
        \item We recognize that providing effective safeguards is challenging, and many papers do not require this, but we encourage authors to take this into account and make a best faith effort.
    \end{itemize}

\item {\bf Licenses for existing assets}
    \item[] Question: Are the creators or original owners of assets (e.g., code, data, models), used in the paper, properly credited and are the license and terms of use explicitly mentioned and properly respected?
    \item[] Answer: \answerYes{}.
    \item[] Justification: Across the manuscript, we have provided all the citations to the original papers for the dataset used. We will point to the specific citation to provide the correct license to each specific case. For our released assets, we provide the license.
    \item[] Guidelines:
    \begin{itemize}
        \item The answer \answerNA{} means that the paper does not use existing assets.
        \item The authors should cite the original paper that produced the code package or dataset.
        \item The authors should state which version of the asset is used and, if possible, include a URL.
        \item The name of the license (e.g., CC-BY 4.0) should be included for each asset.
        \item For scraped data from a particular source (e.g., website), the copyright and terms of service of that source should be provided.
        \item If assets are released, the license, copyright information, and terms of use in the package should be provided. For popular datasets, \url{paperswithcode.com/datasets} has curated licenses for some datasets. Their licensing guide can help determine the license of a dataset.
        \item For existing datasets that are re-packaged, both the original license and the license of the derived asset (if it has changed) should be provided.
        \item If this information is not available online, the authors are encouraged to reach out to the asset's creators.
    \end{itemize}

\item {\bf New assets}
    \item[] Question: Are new assets introduced in the paper well documented and is the documentation provided alongside the assets?
    \item[] Answer: \answerYes{}.
    \item[] Justification: Upon releasing code and a checkpoint on GitHub, we will include instructions for using the released assets.
    \item[] Guidelines:
    \begin{itemize}
        \item The answer \answerNA{} means that the paper does not release new assets.
        \item Researchers should communicate the details of the dataset\slash code\slash model as part of their submissions via structured templates. This includes details about training, license, limitations, etc. 
        \item The paper should discuss whether and how consent was obtained from people whose asset is used.
        \item At submission time, remember to anonymize your assets (if applicable). You can either create an anonymized URL or include an anonymized zip file.
    \end{itemize}

\item {\bf Crowdsourcing and research with human subjects}
    \item[] Question: For crowdsourcing experiments and research with human subjects, does the paper include the full text of instructions given to participants and screenshots, if applicable, as well as details about compensation (if any)? 
    \item[] Answer: \answerNA{}.
    \item[] Justification: The paper does not involve crowdsourcing nor research with human subjects.
    \item[] Guidelines:
    \begin{itemize}
        \item The answer \answerNA{} means that the paper does not involve crowdsourcing nor research with human subjects.
        \item Including this information in the supplemental material is fine, but if the main contribution of the paper involves human subjects, then as much detail as possible should be included in the main paper. 
        \item According to the NeurIPS Code of Ethics, workers involved in data collection, curation, or other labor should be paid at least the minimum wage in the country of the data collector. 
    \end{itemize}

\item {\bf Institutional review board (IRB) approvals or equivalent for research with human subjects}
    \item[] Question: Does the paper describe potential risks incurred by study participants, whether such risks were disclosed to the subjects, and whether Institutional Review Board (IRB) approvals (or an equivalent approval/review based on the requirements of your country or institution) were obtained?
    \item[] Answer: \answerNA{}.
    \item[] Justification: The paper does not involve crowdsourcing nor research with human subjects.
    \item[] Guidelines:
    \begin{itemize}
        \item The answer \answerNA{} means that the paper does not involve crowdsourcing nor research with human subjects.
        \item Depending on the country in which research is conducted, IRB approval (or equivalent) may be required for any human subjects research. If you obtained IRB approval, you should clearly state this in the paper. 
        \item We recognize that the procedures for this may vary significantly between institutions and locations, and we expect authors to adhere to the NeurIPS Code of Ethics and the guidelines for their institution. 
        \item For initial submissions, do not include any information that would break anonymity (if applicable), such as the institution conducting the review.
    \end{itemize}

\item {\bf Declaration of LLM usage}
    \item[] Question: Does the paper describe the usage of LLMs if it is an important, original, or non-standard component of the core methods in this research? Note that if the LLM is used only for writing, editing, or formatting purposes and does \emph{not} impact the core methodology, scientific rigor, or originality of the research, declaration is not required.
    \item[] Answer: \answerNA{}.
    \item[] Justification: LLMs have been used just for writing, editing, or formatting purposes of the manuscript and portions of the experimental source code.
    \item[] Guidelines:
    \begin{itemize}
        \item The answer \answerNA{} means that the core method development in this research does not involve LLMs as any important, original, or non-standard components.
        \item Please refer to our LLM policy in the NeurIPS handbook for what should or should not be described.
    \end{itemize}

\end{enumerate}